\documentclass[letterpaper, 10 pt, conference]{ieeeconf}  

\IEEEoverridecommandlockouts                              

\usepackage{cite}
\usepackage{dsfont}
\usepackage{xspace}

\usepackage{amsmath,amsfonts,bm}

\def\1{\bm{1}}

\def\eps{{\epsilon}}

\DeclareMathAlphabet{\mathsfit}{\encodingdefault}{\sfdefault}{m}{sl}
\SetMathAlphabet{\mathsfit}{bold}{\encodingdefault}{\sfdefault}{bx}{n}

\DeclareMathOperator*{\argmax}{arg\,max}

\renewcommand{\mid}{\,|\,}

\usepackage{color}
\usepackage{comment}
\usepackage{ulem}
\usepackage{algorithmic}
\usepackage[ruled,vlined]{algorithm2e}       
\usepackage{graphicx}
\usepackage{amsmath,amssymb,amsfonts}
\usepackage{algorithmic}
\usepackage{graphicx}
\usepackage{textcomp}
\usepackage{xcolor}
\usepackage{bbm}
\usepackage{wrapfig}
\usepackage{adjustbox}
\usepackage{multicol}
\newtheorem{definition}{Definition}
\newtheorem{theorem}{Theorem}
\newcommand{\alg}{ELF-P\xspace}

\newcommand{\code}[1]{\texttt{#1}}
\usepackage{booktabs}       
\usepackage{multirow}

\newcommand{\OurAppendix}{\href{https://nuria95.github.io/elf-p/}{Appendix} }

\graphicspath{{figures/}}

\usepackage{hyperref}
\usepackage{cleveref}

\definecolor{elfp}{RGB}{225,156,36}
\definecolor{prefill}{RGB}{143,176,50}
\definecolor{ddqn}{RGB}{93,158,199}
\definecolor{ddqnfd}{RGB}{235,98,53}
\definecolor{her}{RGB}{59,58,126}
\definecolor{spirl}{RGB}{76,93,26}
\definecolor{soft-elfp}{RGB}{190, 37, 69}

\title{\LARGE \bf
Efficient Learning of High Level Plans from Play
}

\author{Núria Armengol Urpí$^{1}$, Marco Bagatella$^{1}$, Otmar Hilliges$^{1}$, Georg Martius$^{2}$ and Stelian Coros$^{1}$
\thanks{$^{1}$Department of Computer Science, ETH Zurich, Switzerland {\tt\small \{nuria.armengolurpi,  mbagatella, scoros, otmar.hilliges\}@inf.ethz.ch}}%
\thanks{$^{2}$Max Planck Institute for Intelligent Systems, Tübingen, Germany\newline {\tt\small georg.martius@tuebingen.mpg.de}}
}

\begin{document}

\maketitle
\thispagestyle{empty}
\pagestyle{empty}

\begin{abstract}
Real-world robotic manipulation tasks remain an elusive challenge, since they involve both fine-grained environment interaction, as well as the ability to plan for long-horizon goals. Although deep reinforcement learning (RL) methods have shown encouraging results when planning end-to-end in high-dimensional environments, they remain fundamentally limited by poor sample efficiency due to inefficient exploration, and by the complexity of credit assignment over long horizons.
In this work, we present Efficient Learning of High-Level Plans from Play (\alg), a framework for robotic learning that bridges motion planning and deep RL to achieve long-horizon complex manipulation tasks.
We leverage task-agnostic \textit{play} data to learn a discrete behavioral prior over object-centric primitives, modeling their feasibility given the current context.
We then design a high-level goal-conditioned policy which (1) uses primitives as building blocks to scaffold complex long-horizon tasks and (2) leverages the behavioral prior to accelerate learning.
We demonstrate that \alg has significantly better sample efficiency than relevant baselines over multiple realistic manipulation tasks and learns policies that can be easily transferred to physical hardware.
\end{abstract}

\section{Introduction}
\label{sec:Intro}

One of the collective visions of robotics is a world where robots help humans with daily household chores, such as setting up a table or emptying the dishwasher.
An immediate challenge is that, to operate in the physical world, a robot must be able to reason in a mixed decision space.
That is, it must combine decisions relating continuous motions with discrete subtasks in order to accomplish complex tasks.
For example, the task of emptying the dishwasher requires the robot to first go to the dishwasher, then open it, grab a single dish, open the cupboard and place the dish inside, and repeat until the dishwasher is empty.
Failing to perform a single subtask or ordering them incorrectly would ultimately lead to a failure.
The ability to reason over long horizons, we argue, is thus a crucial milestone in achieving the above-mentioned vision.

Despite significant improvements in several fields of robotics, ranging from motion planning to robust control \cite{winkler2018optimization, winkler2018gait, dai2014whole, bjelonic2022offline},  long-horizon problems involving several tasks still presents a considerable challenge.
While planning algorithms can efficiently search well-defined state spaces and plan over relatively long horizons, they generally struggle with complex, stochastic dynamics, and high-dimensional systems, in particular featuring narrow passages (i.e. low-measure regions which have to be traversed to reach a goal).
Moreover, they often require accurate models of the environment or handcrafted distance measures \cite{sermanet2021broadly, eysenbach2019search, mainprice2020interior}.

Learning methods have thus emerged as a more scalable approach for handling the high-dimensionality of the real world.
While capable of obtaining hard-to-engineer behaviors \cite{kalashnikov2021mt, akkaya2019solving, andrychowicz2020learning, kalashnikov2018qt}, learning-based methods are similarly
limited to short horizons tasks in general, since longer plans involving multi-level reasoning aggravate the challenges of exploration and temporal credit assignment \cite{nachum2018data}.

\begin{figure}[t]
\begin{center}
\includegraphics[width=\columnwidth]{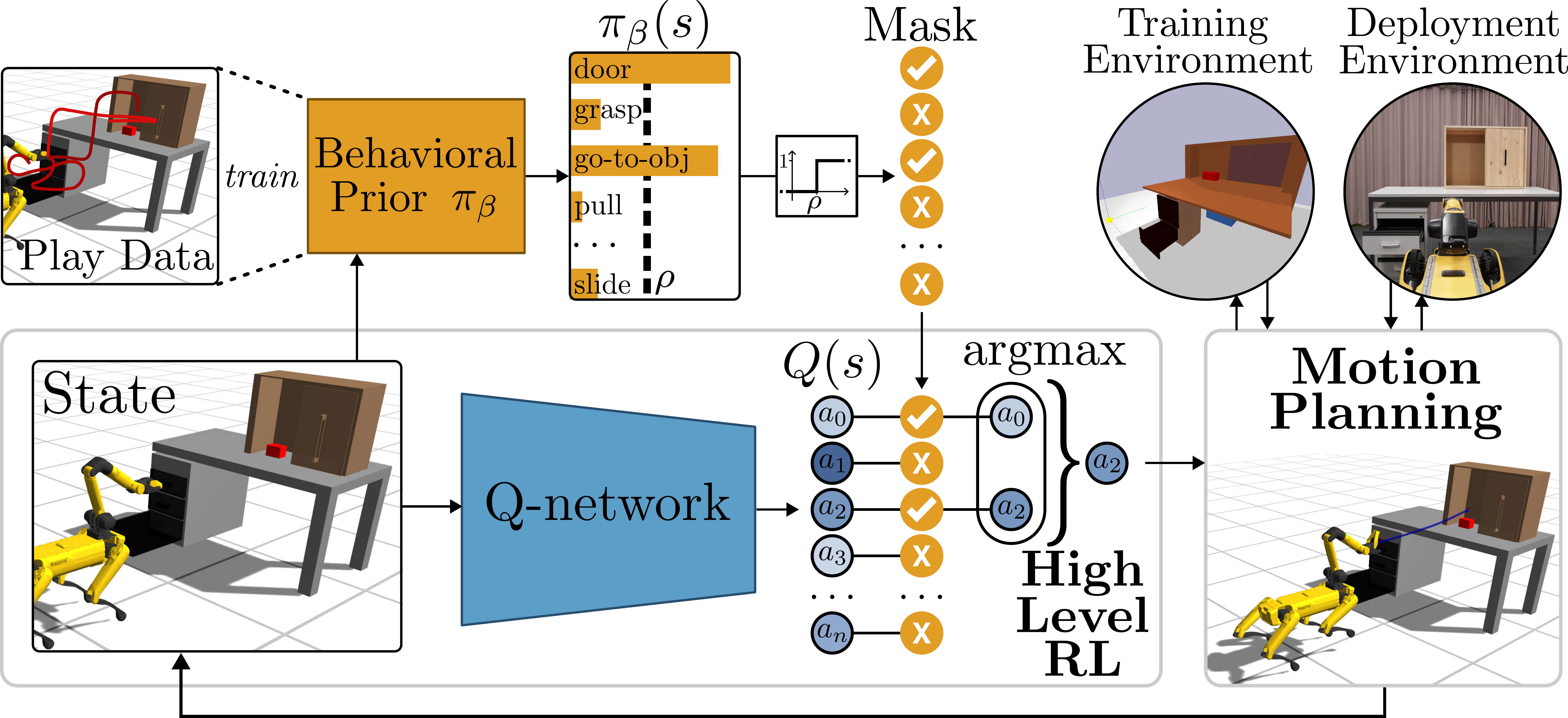}
\end{center}
\vspace{-4mm}
\caption{Overview of \alg. A discrete behavioral prior learned on play data is used to discard infeasible actions at every state, enabling a high-level RL agent to search over a smaller space of behaviors. Each action is mapped to a motion primitive, which is executed through classical motion planning methods. After training in simulation, the policies can be easily deployed on real hardware.}
\vspace{-5mm}
\label{fig:method}
\end{figure}
In this work, we introduce \alg, a novel approach for robotic learning of long-horizon manipulation tasks, bridging motion planning and deep reinforcement learning (RL).
In a two-level hierarchy, high-level planning is entrusted to a goal-conditioned policy, trained through deep RL and designed to deal with the high-dimensionality and complexity of real-world dynamics. This policy operates at a coarser time scale and controls the execution of primitives, which lie at the lower level in the hierarchy (see Figure \ref{fig:method}). Crucially, we sidestep the instability of jointly learning high and low-level policies and thus rely on a predefined library of object-centric primitives (such as grasping or reaching a target configuration in a collision-free path) obtained through classical motion planning.

While planning over primitives substantially reduces the burden of the higher level, uniform exploration may still be insufficient in long-horizon tasks with sparse rewards, as the space of possible behaviors grows rapidly with the size of the action space and number of steps to the goal \cite{levy2018hierarchical}.
Although humans excel at inferring the set of behaviors enabled by a situation, uninformed RL agents tend to repeatedly attempt futile actions, resulting in poor sample efficiency and potentially unsafe exploration.

We thus additionally propose to prune the high-level action space to a subspace of feasible primitives, inferred dynamically at each step. In practice, this is modelled through a discrete behavioral prior learned from task-agnostic play data.

While existing works on behavioral priors largely rely on soft integration schemes, we instead propose to discretize the prior into a hard, binary constraint. We then define an algorithm that only explores and exploits over the set of feasible actions, effectively performing Q-learning in a reduced Markov Decision Process (MDP).
Theoretically, a near-optimal solution to this MDP can be retrieved more efficiently and under mild assumptions, can be generalized to the original MDP.
Empirically, we observe improved performance in a variety of complex manipulation tasks and highlight the ease of transferring the policy to real-world hardware.

Our approach represents a promising step towards learning to solve long-horizon tasks in a hierarchical manner by combining task-agnostic play data, RL and motion planning algorithms. Our contributions can be summarized as follows: (a) we explore the integration of RL with classic motion planning algorithms and show its feasibility in a simulated manipulation environment; (b) we highlight how the exploration problem can be further addressed by introducing a prior model learned from play data; (c) we evaluate our proposed algorithm both theoretically and empirically, while comparing it against relevant baselines; (d) we show how learned policies can be deployed to physical hardware.
Videos and additional information can be found at \href{https://nuria95.github.io/elf-p/}{\tt nuria95.github.io/elf-p}.
\nocite{coumans2021, kingma2014}

\section{Related Work}
\label{sec:rel_work}

\subsection{Planning Algorithms}
Classical motion planning methods to move a robot in a collision-free path are based on sampling \cite{kavraki1996probabilistic, lavalle2001randomized, karaman2011sampling} or constrained optimization \cite{ratliff2009chomp, schulman2014motion} and can efficiently find long-horizon paths. However, they do not allow a robot to alter the world. Task-and-motion planning (TAMP) methods are used when planning for a robot that operates in environments containing a large number of objects, taking actions to navigate the world as well as to change the state of the objects \cite{kaelbling2013integrated, toussaint2015logic, toussaint2017multi,garrett2020integrated, braun2021rhh}.
Despite performing well in long-horizon planning, they are limited by the need of a world model, their task representations, the dimensionality of the search space and the inability to execute high-dimensional, complex tasks robustly.

\subsection{Model-Free Learning}
Model-free learning offers a promising alternative when dealing with unknown environments but relies on a reward function defining the task at hand.
Additionally, model-free methods are sample inefficient \cite{duan_one-shot_2017} in general and have difficulties reasoning over long horizons.
Hierarchical RL (HRL) \cite{barto2003recent, dayan1992feudal, wiering1997hq, sutton1999between,dietterich2000hierarchical} methods have been proposed as a scalable solution that directly leverages temporal abstraction but in practice, they struggle with sample complexity and suffer from brittle generalization \cite{nasiriany2022augmenting}.
Several HRL algorithms rely on goal-conditioned policies \cite{Kaelbling93learningto, andrychowicz2017hindsight, gupta_relay_2019, christen2020hide} for low-level control but they tend to make training unstable (see discussion in \cite{nachum2018data}).
To address this, several works have adopted classical motion planning techniques to replace the high-level policy (e.g. tree search methods \cite{eysenbach2019search,sermanet2021broadly}). In contrast, we sidestep the problem by replacing the low-level policy with a predefined set of task-primitives and train a model-free high-level policy to control their execution.
Closely related to our work, \cite{ICLR16-hausknecht, dalal2021accelerating,nasiriany2022augmenting} use Parameterized Action MDPs \cite{masson2016reinforcement} in which the agent executes a parameterized primitive at each decision-making step. These methods also rely on a specified library of primitives but, unlike ours, they are parameterized. However, this versatility in primitive instantiation comes at a low efficiency cost since the agent needs to explore large amounts of parameters to solve long horizon tasks. Other methods also use motion planners as a low-level controller and learn an RL policy in a high-level action space \cite{relmogen2021, yamada21a, angelov2020}.

\subsection{Guiding Exploration in Large Action Spaces}
As exploration remains a fundamental challenge for RL agents, previous works have attempted to ease the burden of planning in complex spaces by pruning the actions available at each step.
Invalid action masking has been proposed in large strategy games to restrict sampling to a fixed subset of the action space  \cite{vinyals2017new, berner2019dota, ye2020mastering}. However, these methods assume that the set of illegal actions is given a priori. The case in which a random subset of all actions is available at each step was also formally studied by \cite{boutilier2018planning}.
When prior information is not available, several methods naturally propose to learn which actions are suitable. In the action elimination literature, this is achieved with \cite{zahavy2018learn} or without  \cite{even2003action} assuming the availability of additional supervision from the environment, although the latter case remains constrained to tabular settings.
Another relevant framework is that of affordances \cite{gibson1977theory, khetarpal2020what, costales2022possibility}, which measure the possibility for single actions to achieve desired future configurations. Affordances can again be learned from an explicit signal from the environment \cite{khetarpal2020what} or utilize prior information. 
Our method is perhaps more closely related to behavioral priors, which can be learned from offline trajectories \cite{pertsch2020spirl, pertsch2021skild, singh2021parrot} or online interaction \cite{tirumala2022behavior} and can be used to direct exploration \cite{singh2021parrot} or regularize learned policies \cite{tirumala2022behavior}. Behavioral priors can be modeled as conditional probability distributions over the action space. Our method also learns a state-conditional model of possible actions from data and does not require explicit supervision (i.e. validity labels for the agent's actions).
However, while behavioral prior literature largely focuses on continuous actions and soft constraints (e.g. KL-regularization), our method learns over a discrete set of actions and relies on a hard integration of the prior into the training process.

\section{Problem Statement}
\label{sec:statement}
We model the environment as a goal-conditioned Markov Decision Process represented as an 8-tuple $(\mathcal{S}, \mathcal{G}, \mathcal{A}, P, R, \rho_0, \rho_g, \gamma)$ with possibly continuous states $s \in \mathcal{S}$ and goals $g \in \mathcal{G}$, discrete actions $a \in \mathcal{A}$, transition kernel $P(\cdot|s,a)$, reward function $R(s, g)$, initial state distribution $\rho_0$, goal distribution $\rho_g$ and discount factor $\gamma$. We focus on sparse reward signals: $R(\cdot|s, g) = \mathds{1}_{\{|f(s)-g|_d < \epsilon\}}$, where $|\cdot|_d$ is an arbitrary distance metric, $f: \mathcal{S} \to \mathcal{G}$ is a projection and $\epsilon$ is a threshold.
We additionally denote by $\pi: \mathcal{S \times G \to A}$ a goal-conditioned stationary policy.

Crucially, we assume access to a fixed dataset collected through \textit{play}, following an unknown behavior policy $\pi_\beta$.
Play data \cite{lynch_learning_2019} can be inexpensively collected by a human operator controlling a robot to achieve arbitrary environment configurations while interacting with the objects at hand.
This is useful to extract \textit{affordances} \cite{gibson1977theory}, representing the subset of possible actions which are feasible in a current situation. For example, a drawer can be pushed or pulled, but not moved sideways.
In our framework, \textit{play} happens at a high level, more specifically at the level of motion primitives or skills, where a human can choose from a predefined set of primitives, which are generally available (e.g. precoded behaviors for real-word hardware such as for Spot robot \cite{noauthor_spot_nodate} or as in \cite{saycan2022arxiv}).
Another solution, which could involve recording low-level action sequences, and segmenting them into meaningful skills as in \cite{zhu2022bottom,konidaris2010constructing, niekum2015online}, is left for future works.
The resulting play dataset consists of $N$ high-level state-action pairs $\mathcal{D} = \{(s_i, a_i)\}_{i=1}^N$ where a high level action $a_i$ represents the primitive enacted by the play agent in state $s_i$.

Following the goal-conditioned framework \cite{schaul2015universal}, we intend to find a policy $\pi: \mathcal{S \times G \to A}$ that maximizes the expected discounted return
\begin{equation}
J(\pi) = \mathbb{E}_{g \sim \rho_g, \mu^\pi}\left[\sum_{t=1}^\infty \gamma^{t-1}R(s_t, g)\right],
\end{equation}
under the trajectory distribution \mbox{$\mu^\pi (\tau | g) = \rho_0 (s_0) \prod_{t=0}^{\infty} P(s_{t+1} | s_t,a_t)$ with $a_t=\pi(s_t, g)$}.

We remark that through this paper, the notation of atomic actions $a \in \mathcal{A}$ is used to represent non-parametric motion primitives. We effectively abstract away the complexity of low-level control and let the action space $\mathcal{A}$ be composed of a finite set of high-level behaviors which are, in practice, executed through established motion planning methods. As a consequence, a single timestep in the MDP corresponds to the execution of a single primitive. We finally remark that, while the method described in Section \ref{sec:method} is designed and deployed in this particular setting, it remains generally applicable (e.g. in MDPs with low-level actions).
\section{Efficient Learning of High Level Plans from Play (\alg)}
\label{sec:method}

We now present our algorithm \alg for solving long-horizon tasks with motion primitives and play-guided RL.
To address exploration, we first learn a discrete behavioral prior that eliminates infeasible actions from the set of primitives and hence prunes the search space (\Cref{ssec:learn_prior}, Figure \ref{fig:spot_priors}).
Next, in \Cref{ssec:learn_maskedQ1} we propose and motivate an integration scheme for the learned prior onto Q-learning. Consequently, our agent can focus on learning Q-values for feasible state-action pairs only, as the prior generally lifts the burden of learning to avoid infeasible actions.

\begin{figure}[b]
\begin{center}
\vspace{-1mm}
\includegraphics[width=\columnwidth]{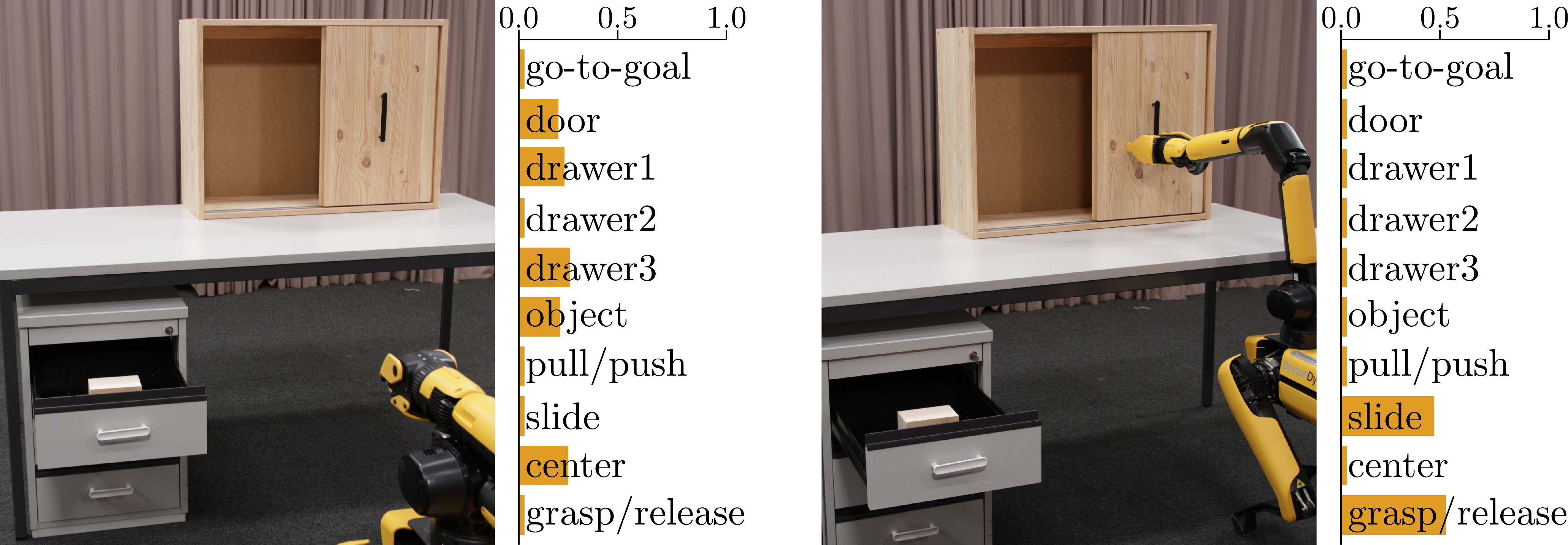}
\end{center}
\vspace{-2mm}
\caption{The trained behavioral prior $\pi^\beta$ learns to estimate the set of feasible primitives in different environment configurations. (Left) When the agent is at the center of the space, the prior favors primitives that involve reaching elements through a free-collision path (go to \textit{door}, \textit{drawer}s or above the \textit{object}), but prevents actions that involve object manipulation. (Right) When being close to the door, the prior learns correct object affordances such as \textit{grasp} or \textit{slide}.}
\vspace{-5mm}
\label{fig:spot_priors}
\end{figure}

\subsection{Learning a Prior from Play} \label{ssec:learn_prior}

Let us start by considering the play dataset \mbox{$\mathcal{D} = \{ (s_1, a_1),  ... (s_N, a_N) \}$} introduced in Section \ref{sec:statement}. While lacking explicit exploitative behaviors, play data inherently favors actions that are \textit{feasible}: while not necessarily desirable for a given goal, those actions are likely to be successfully executed given the current state of the environment.
We aim to extract this information by estimating a (goal-independent) behavioral prior $\pi^\beta(\cdot|s)$, modeled as a conditional categorical distributions over primitives, which associates feasible primitives to higher likelihoods.

We parameterize $\pi^\beta$ through a neural network with learnable parameters $\omega$, which can be trained via standard mini-batch first order techniques to minimize the negative log-likelihood $\mathcal{L}_{NLL}(\omega)$:
\begin{equation}
    \mathcal{L}_{NLL}(\omega) = \mathop{\mathbb{E}}_{\substack{B \sim \mathcal{D}}}\Bigg[\frac{1}{|B|} \sum_{\substack{(s, a) \in B}} -\log \pi^\beta_\omega(a|s) \Bigg],
\label{eq:prior_loss}
\end{equation}
where $B$ represents a batch of state-action pairs sampled uniformly from the dataset $\mathcal{D}$.

\subsection{Selecting Feasible Actions}
Given the learned behavioral prior $\pi_\beta(\cdot|s)$, we propose to turn its soft probability distribution into hard, binary constraints on the action space. We thus define a threshold-based selection operator $\alpha: \mathcal{S} \to \mathcal{P(A)}$:
\begin{equation}
 \alpha(s) = \{ a \in \mathcal{A} \mid \pi_\beta(a|s) > \rho\},
\label{eq:mask}
\end{equation}
where $\mathcal{P(\cdot)}$ represents a powerset.
Ideally, an action $a \not \in \alpha(s)$ would not be chosen by an optimal goal-conditioned policy in state $s$. We refer to $\alpha(s)$ as the set of \textit{feasible} actions for state $s$. See Figure \ref{fig:spot_priors} for a visualization.

\subsection{Learning in a Reduced MDP}\label{ssec:learn_maskedQ1}

The learned selection operator $\alpha(s)$ enables the definition of an auxiliary MDP $\mathcal{M'}$, which we refer to as \textit{reduced} MDP. The definition and solution of the reduced MDP lay at the core of our method.
We model $\mathcal{M'}$ through a generalized definition of MDPs \cite{puterman1994} in which available actions depend on the current state: in the state $s$, the action space is restricted to a subset $\alpha(s) \subseteq \mathcal{A}$.

\begin{definition}[Reduced MDP]
  Given an MDP \mbox{$\mathcal{M}=(\mathcal{S}, \mathcal{G}, \mathcal{A}, P, R, \rho_0, \rho_g, \gamma)$} and a selection operator $\alpha: \mathcal{S} \to \mathcal{P(A)}$ such that for all $s \in \mathcal{S}$, $\alpha(s) \neq \emptyset$, the reduced MDP $\mathcal{M'}$ is defined as the 9-tuple $(\mathcal{S}, \mathcal{G}, \mathcal{A}, P, R, \rho_0, \rho_g, \alpha, \gamma)$.
\end{definition}
\vspace{2mm}
where the assumption on $\alpha(s)$ ensures that there exist a feasible action in each state and Q-values can be well-defined.
Intuitively, $\mathcal{M'}$ encodes the same environment as $\mathcal{M}$ but restricts the set of action-state pairs.

We note that in PAC RL settings, the analysis of sample complexity \cite{kakade2003sample} (i.e. the number of steps for which a learned policy is not $\epsilon$-optimal with high probability) produces upper bounds that are directly dependent on the number of state-action pairs \cite{lattimore2012pac}.
In particular, model-free PAC-MDP algorithms can attain a sample complexity that is $\tilde O(N)$, where $N \leq |\mathcal{S}||\mathcal{A}|$ is the number of state-action pairs, and $\tilde O(\cdot)$ represents $O(\cdot)$ where logarithmic factors are ignored \cite{strehl2006pac}.
Learning in $\mathcal{M'}$ instead of $\mathcal{M}$ is thus desirable and could lead to near-linear improvements in sample efficiency as the number of infeasible actions grows.
Crucially, under mild assumptions, the optimal policy for $\mathcal{M'}$ can not only be retrieved more efficiently but also attains optimality in the original MDP $\mathcal{M}$ (see \OurAppendix \ref{app:optimality}).

We thus propose a practical modified Q-learning iteration on the original MDP $\mathcal{M}$, which is equivalent to performing Q-learning directly in the reduced MDP $\mathcal{M'}$.
Given a transition $(s, a, s', g, r)$:
\begin{equation}
    Q(s, a, g) \gets (1-\delta)Q(s,a,g) + \delta(r + \gamma \max_{a' \in \alpha(s)}Q(s', a', g)),
\end{equation}
where the value of the next-state $s'$ is only computed over feasible actions and $\delta$ is the learning rate.

Under common assumptions (i.e. infinite visitation of each state-action pair and well-behaved learning rate in tabular settings \cite{bertsekas1996neuro}), this algorithm converges to $Q_{\mathcal{M'}}^*$, from which we can easily extract $\pi_{\mathcal{M}}^*(s,g) = \pi_{\mathcal{M'}}^*(s,g) = \argmax_{a \in \alpha(s)} Q_{\mathcal{M'}}(s,a,g)$. For simplicity, we will from now on refer to  $Q_{\mathcal{M'}}$ as $Q$.

In practice, following the goal-conditioned framework \cite{schaul2015universal}, we scale this algorithm by parameterizing $Q_\theta(s,a,g)$ through a neural network.
Inspired by recent success in scaling Q-learning \cite{watkins1992q} to high-dimensional spaces \cite{van2016deep, mnih2015human, fujimoto2018addressing} while reducing overestimation bias, we learn the parameters $\theta$ of the Q-function using Clipped Double Q-learning \cite{fujimoto2018addressing}, which minimizes the temporal difference (TD) loss:
\begin{gather}
    \mathcal L (\theta_j) = \mathbb E_{\substack{(s,a,s',g,r) \sim \mathcal{B} }} \big [(y_j - Q_{\theta_j}(s_t,a_t,g))^2]\label{eq:q-loss}, \text{ with } \\
    \nonumber y_j =  r + \gamma \min_{i=1,2}  Q_{\theta'_i}(s', \argmax_{a_{t+1} \in \alpha(s_{t+1})}  Q_{\theta_j}(s_{t+1}, a_{t+1}, g), g ),
\end{gather}
where $j \in \{1,2\}$, and $\theta_j, \theta_j'$ are the parameters for Q and target Q-networks respectively and where $(s,a,s',g,r)$ tuples are sampled uniformly from an experience replay buffer \cite{lin1992reinforcement} $\mathcal{B}$ exploiting the off-policy nature of Q-learning.
We collect experience following an $\eps$-greedy exploration mechanism on the feasible action set $\alpha(s) \subseteq \mathcal{A}$. We summarize our approach in Algorithm \ref{algorithm}.

\setlength{\textfloatsep}{0mm}
\begin{algorithm}[ht]
 \caption{\alg} \label{alg:algorithm}
\begin{algorithmic}
\small
    \INPUT Trained prior $\pi^\beta_w$ , randomly initialized $Q_\theta$ and Q-target $Q_{\theta'}$ with $\theta = \theta'$, probability threshold $\rho$, learning rate $\eta$, replay buffer $\mathcal{D}=\emptyset$, soft update parameter $\mu$.
    \FOR{episode$=1,\ldots$ N}
        \STATE Sample $s \sim \rho_0$, $g \in \rho_g$.
        \FOR{step$=1,\ldots T$}
        \STATE Compute feasible action set $\alpha(s_t)$ in \eqref{eq:mask}, compute $Q_\theta(s_t, a, g)$ for each $a \in \mathcal{A}$.
        \STATE With probability $\epsilon$ sample $a_t \sim \mathcal{U}\{\alpha(s_t)\}$, else select $a_t = \argmax_{a \in \alpha(s_{t+1})}Q_\theta(s_t, a, g)$.
        \STATE Execute $a_t$ and store transition $(s_t, a_t, r_t, s_{t+1}, g)$ in $\mathcal{D}$.
        \STATE Sample minibatch of transitions $(s_j, a_j, r_j, s_{j+1}, g)$ uniformly from $\mathcal{D}$.
        \STATE Compute TD loss $\mathcal L_{\mathrm{Q}} (\theta)$ in \eqref{eq:q-loss}.
        \STATE Gradient step $\theta \gets \theta - \eta \nabla \mathcal L_{\mathrm{Q}}(\theta)$.
        \STATE Perform soft-update on $\theta' \gets \mu \theta + (1-\mu) \theta' $.
        \ENDFOR
    \ENDFOR
\end{algorithmic}
\label{algorithm}
\end{algorithm}
\section{Results}\label{sec:Experiments}

In this Section, we empirically evaluate the performance of \alg against several baselines. We investigate whether \alg can retrieve optimal policies more efficiently than existing approaches that leverage prior data, and whether the number of infeasible actions that are attempted throughout the training process is reduced.
We evaluate \alg on a variety of manipulation tasks in a simulated environment and then deploy the learned policies on real hardware to evaluate the ease of transfer.
For further details on the environment, additional experiments and extended results, see \OurAppendix \ref{app:env}, \ref{app:experiments}, \ref{app:additional_exp_details} and \ref{app:hardware_exp}.

\begin{figure}[b]
\vspace{2mm}
\begin{center}
\includegraphics[width=\columnwidth]{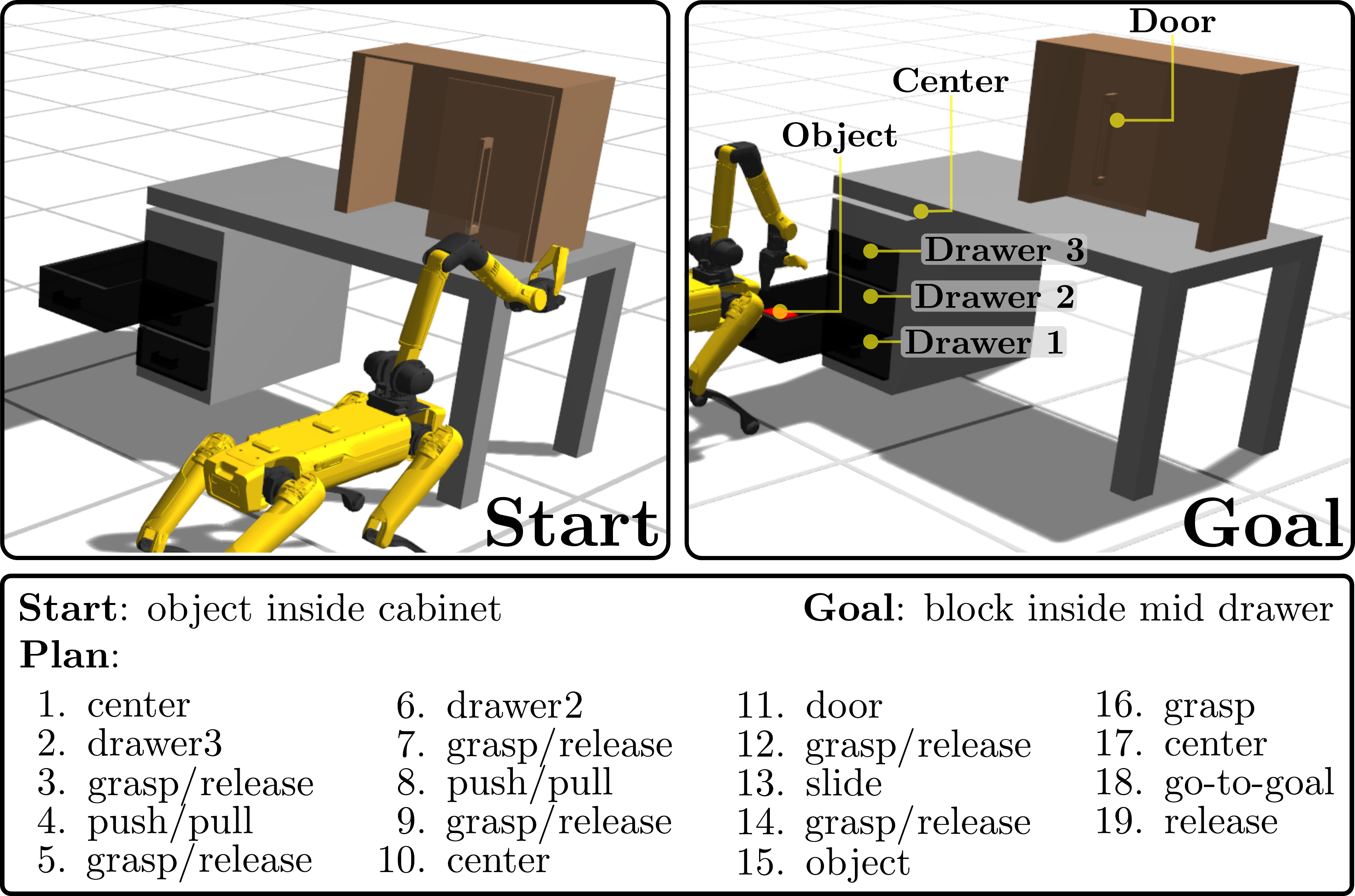}
\end{center}
\vspace{-2mm}
\caption{An example of a task.
Starting from an arbitrary configuration, the agent needs to place the block, which is hidden behind the cabinet door, in the mid drawer, which is blocked by the drawer above. The sequence of 19 actions required to achieve the goal is shown below. }
\label{fig:task_example}
\end{figure}

\begin{figure*}[!htb]
\begin{center}
\vspace{0mm}

\minipage{\textwidth}
\small
\centering
\textcolor{elfp}{\rule[2pt]{20pt}{3pt}} \textrm{ELF-P} \quad
\textcolor{ddqn}{\rule[2pt]{20pt}{3pt}} \textrm{DDQN} \quad
\textcolor{her}{\rule[2pt]{20pt}{3pt}} \textrm{DDQN+HER}
\textcolor{prefill}{\rule[2pt]{20pt}{3pt}} \textrm{DDQN+Prefill} \quad
\textcolor{ddqnfd}{\rule[2pt]{20pt}{3pt}} \textrm{DDQNfD} \quad
\textcolor{spirl}{\rule[2pt]{20pt}{3pt}} \textrm{SPIRL}
\endminipage
\vspace{-1mm}
\end{center}
\minipage{\linewidth}
\includegraphics[width=0.24\textwidth]{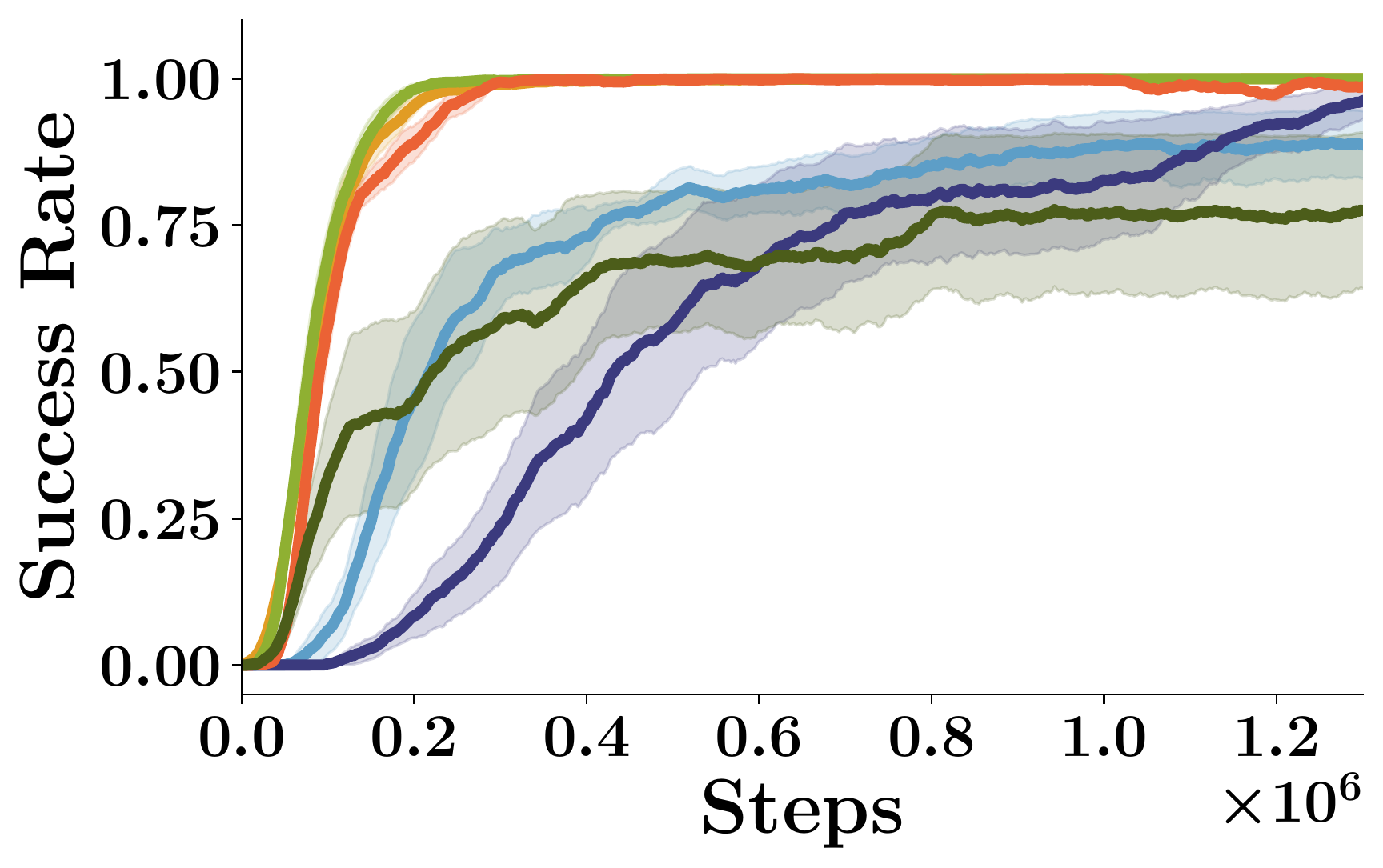}
\label{fig:medium_success}
\includegraphics[width=0.24\textwidth]{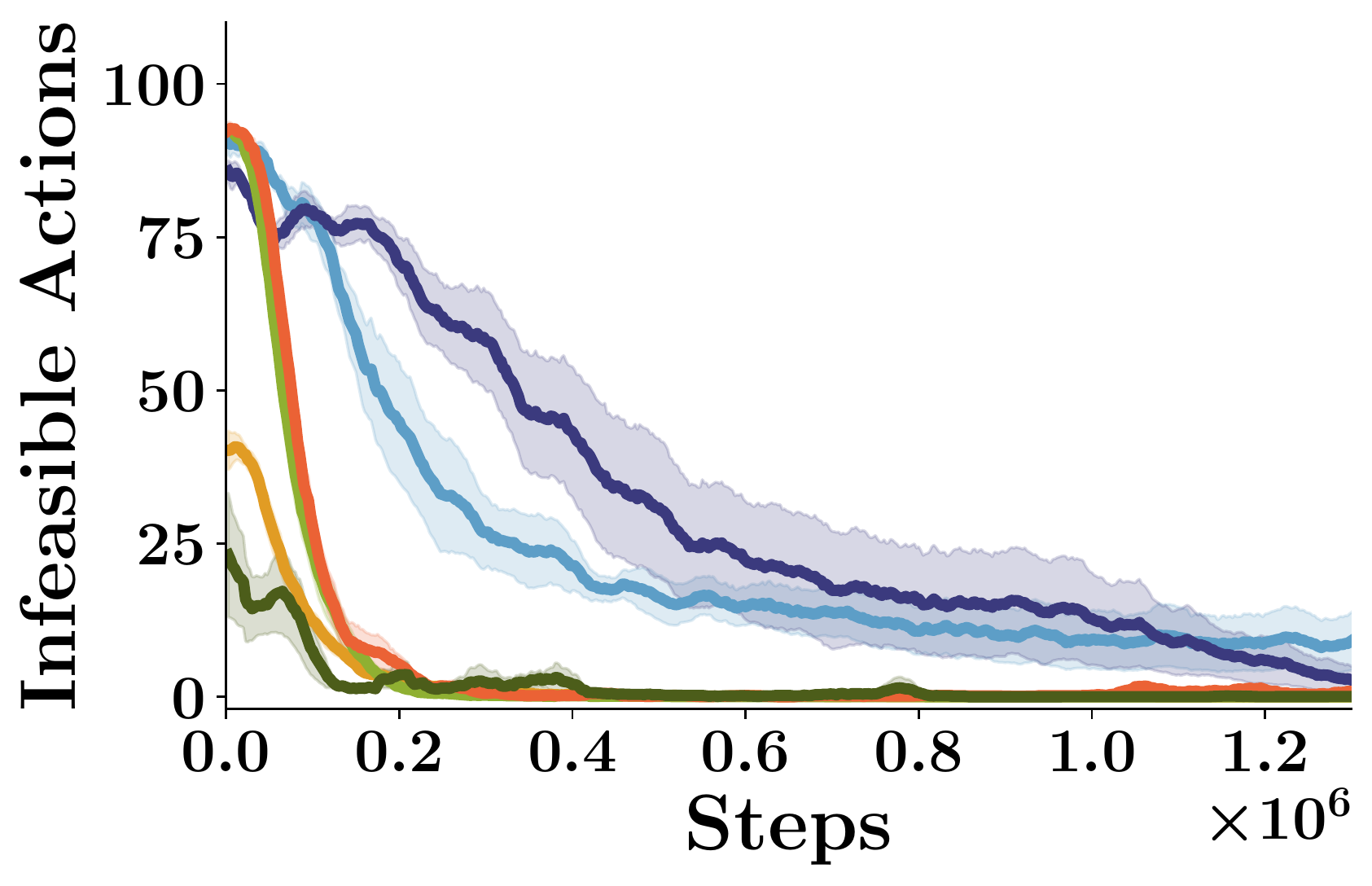}
\label{fig:medium_inv}
\includegraphics[width=0.24\textwidth]{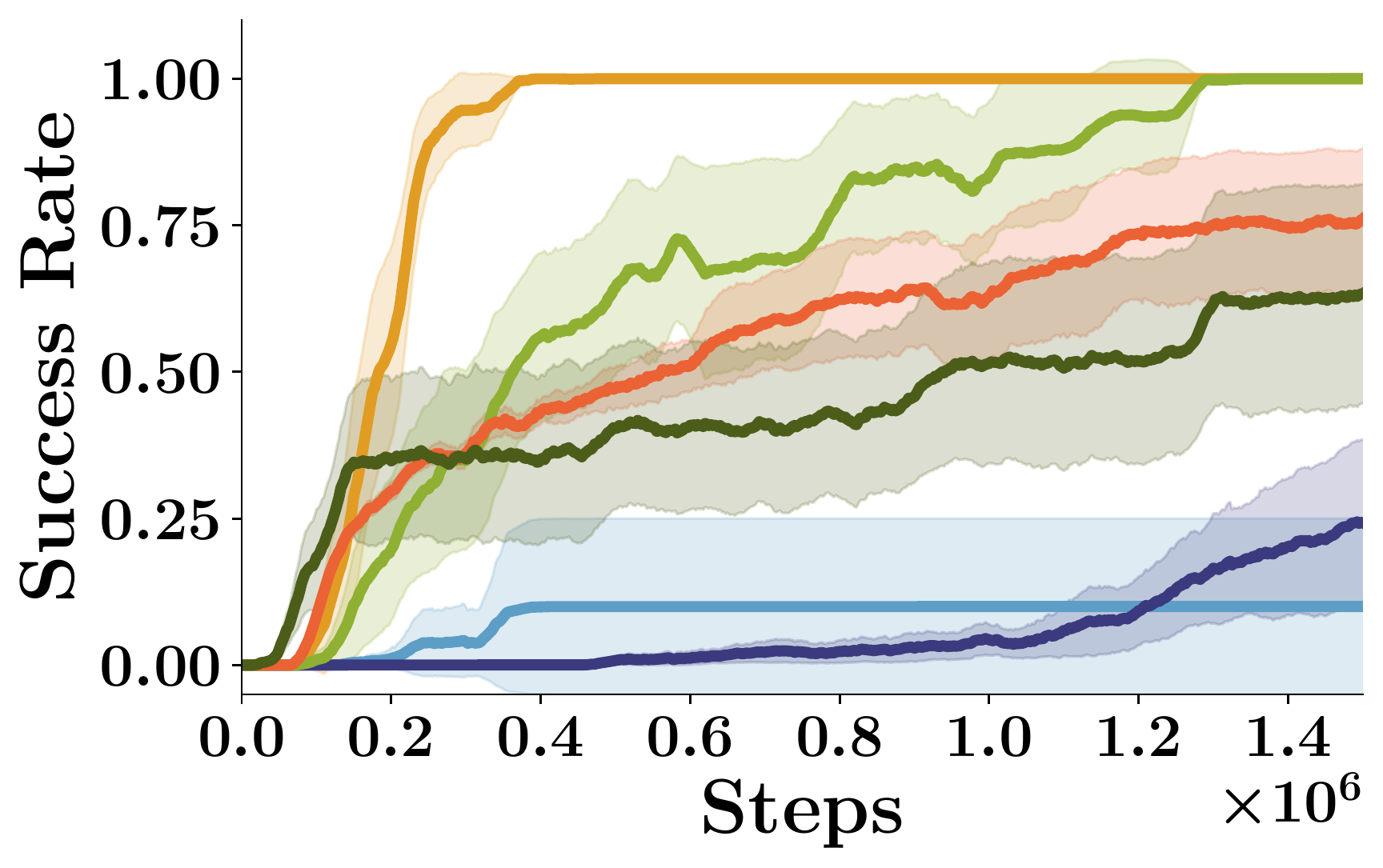}
\label{fig:hard_success}
\includegraphics[width=0.24\textwidth]{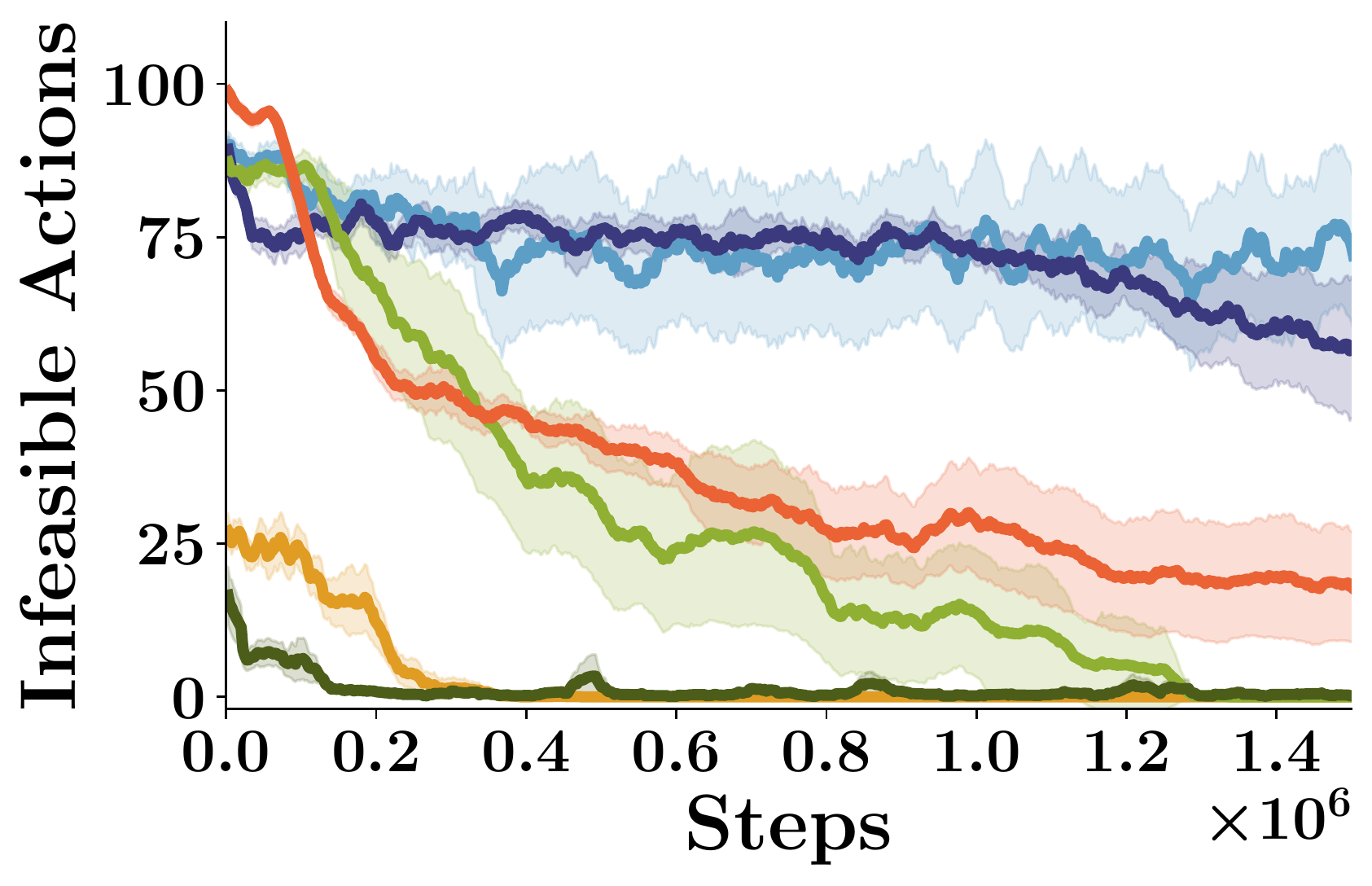}
\label{fig:hard_inv}
\endminipage
\vspace{-5mm}
\caption{Success rate and number of infeasible actions attempts for the Medium (left) and Hard (right) task variants. (Results are averaged across 10 independent random seeds, shaded area represents standard deviation).}
\label{fig:simulation-results}
\vspace{-0.6cm}
\end{figure*}

\begin{figure}
\begin{center}
\vspace{2mm}

\minipage{\linewidth}
\small
\centering
\textcolor{elfp}{\rule[2pt]{20pt}{3pt}} \textrm{ELF-P} \quad
\textcolor{soft-elfp}{\rule[2pt]{20pt}{3pt}} \textrm{SOFT ELF-P}
\endminipage
\vspace{-1mm}
\end{center}
\minipage{\linewidth}
\includegraphics[width=0.49\textwidth]{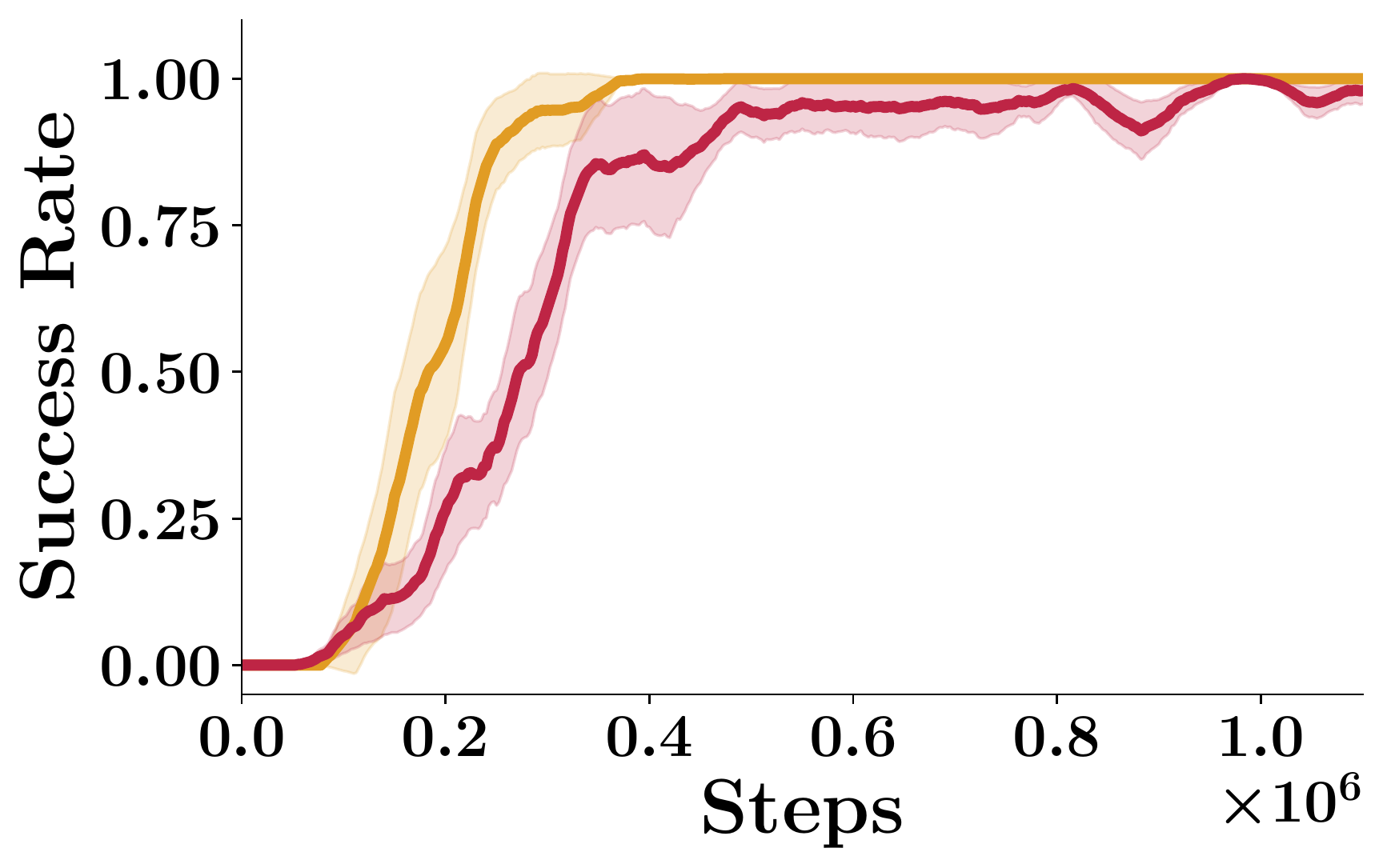}
\label{fig:hard_soft_success}
\includegraphics[width=0.49\textwidth]{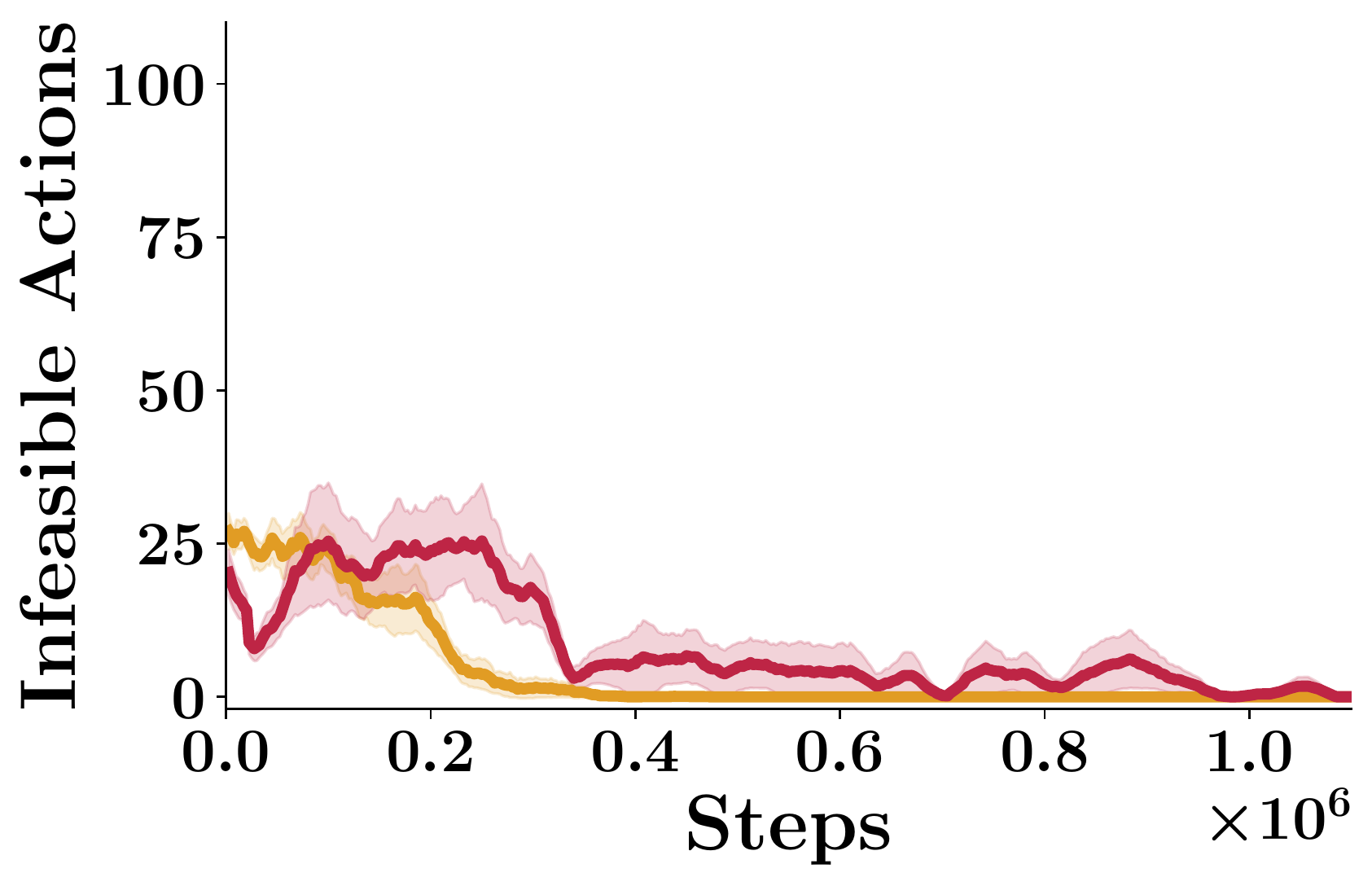}
\label{fig:hard_soft_inv}
\endminipage
\vspace{-5mm}
\caption{Comparison between \alg and Soft-\alg. \alg. Success rate and number of infeasible actions attempts for the Hard  task variant. (Play dataset size $10^4$). Results are averaged across 10 independent
random seeds, shaded area represents standard deviation) }
\label{fig:soft_vs_hard}
\end{figure}
We compare \alg with an unmodified DDQN \cite{van2016deep}, with DDQN with Hindsight Experience Replay (DDQN+HER) \cite{andrychowicz2017hindsight},
with two other state-of-the-art methods  that can leverage prior data, namely DDQNfD \cite{hester2018deep} (with 1-step returns and unprioritized experience replay for a fair comparison), DDQN+Prefill, which initializes its replay buffer with the task-agnostic data $\mathcal{D}$, and with SPiRL \cite{pertsch2020spirl}, which  also uses play data to constrain exploration, but adopts an actor-critic framework with soft prior regularization, which we adapted to operate over discrete action spaces. Extended details are in \OurAppendix \ref{app:baselines}.

\subsection{Simulation Experiments} \label{exp:simulation}
We train and evaluate the algorithm on a variety of simulated long-horizon manipulation tasks in which a robot needs to interact with a realistic desk environment featuring several household items. Possible behaviors include sliding a cabinet door, opening several drawers and moving a wooden block (see Figure \ref{fig:task_example}).
The tasks are defined in the same environment that play data was obtained in. While play data can be collected by humans, for simplicity we use a scripted policy (details are in \OurAppendix \ref{app:additional_exp_details}). 
The goal of each task is to place the block in an arbitrary desired position, which generally also requires manipulating the rest of the items. The episode ends when exceeding a predefined number of steps or when the goal is achieved. The reward function is sparse and is equal to one if and only if the task is completed in time, otherwise a reward of zero is given.

We evaluate on tasks distributions with two levels of difficulty: \textit{Medium} (M) and \textit{Hard} (H). The average number of actions required to solve the (M) tasks for an expert planner is 14, with the longest task requiring 16 steps. For the (H) tasks, the average is 23 steps, with the longest task requiring 29 steps. See Figure \ref{fig:task_example} for an example.

\vspace{2mm}
\subsubsection{Performance Analysis}

Figure \ref{fig:simulation-results} shows the success rate and the number of infeasible actions attempts averaged over 50 evaluation episodes across the two task distributions.
In the (M) tasks \alg shows the same sample efficiency as the best baseline, while in the (H) tasks it is significantly better than all competitors.
For both (M) and (H) tasks, the number of infeasible actions that are attempted by the agent is significantly lower than other baselines. We note that the execution of infeasible actions is due to inaccuracies in the trained prior.

While vanilla DDQN cannot master the (M) task and fails to solve the (H) task, Prefill+DDQN and DDQNfD manage to solve the (M) task as efficiently as \alg, proving that having access to task-agnostic play data is beneficial for the learning process. However, when the action space grows in size, their performance decreases significantly.
We also notice that although managing to reach over 0.5 success rate for both tasks, SPIRL shows a much lower performance than our method and most competitive baselines, hence integrating a soft prior via KL-regularization might not be beneficial in this setting.
Finally, we observe that adding HER relabeling helps slightly on the (H) tasks but hurts performance on the (M) tasks.
We report that using HER with \alg also hurts performance (see \OurAppendix \ref{app:experiments}). HER relies on a gradual growth of the frontier of achieved goals, which can be used for relabeling. Since the dynamics in our environment are not smooth (i.e. a single action often leads to large changes in the state), we hypothesize that HER cannot interpolate to unseen goals and hence hurts performance.
This behavior was also pointed out by \cite{eysenbach2019search}.

We note that all methods in this section that can leverage prior data have access to a play dataset with $10^4$ datapoints, which simulates an amount of data that could in practice, be collected by a few human operators.

\vspace{2mm}
\begin{wrapfigure}{r}{44mm}
\vspace{-3mm}
\centering
\includegraphics[width=\linewidth]{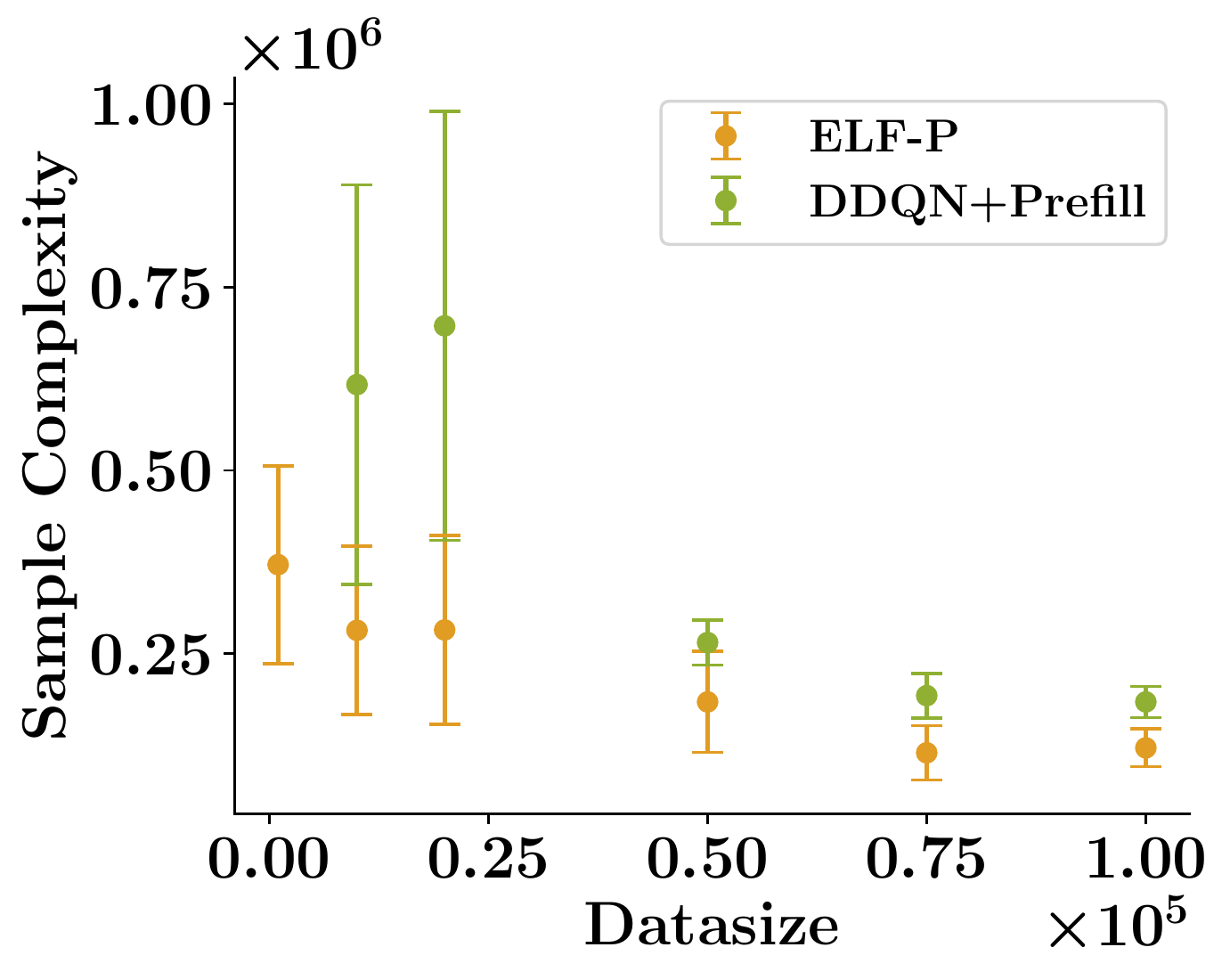}
\vspace{-6mm}
\caption{Effect of play dataset size on sample complexity (DDQN+Prefill fails to solve the task with 1000 datapoints).}
\vspace{-3mm}
\label{fig:boxplot_dataset size}
\end{wrapfigure}
\subsubsection{Robustness to Play Dataset Size}
\label{sec:robustness-datasize}
We study the effect of play dataset size on training performance for both \alg and DDQN+Prefill (which affects the amount of data used for training the prior for \alg and the amount of data used to prefill the replay buffer for DDQN+Prefill). In particular, we measure the number of timesteps required to reach a success rate of 0.95 on (H) tasks when using different dataset sizes.
We report results in Figure \ref{fig:boxplot_dataset size}. We observe that our method retains most of its performance when the prior is trained on minimal quantities of data whereas DDQN+Prefill struggles in low-data regimes. This shows that our method is more suited for the setting of \textit{play}, and is able to learn efficiently with reasonable amounts of data, i.e., data that could in practice be collected by a few human \mbox{operators ($\sim$ 1h30min} of interaction data).
When more data is available, we generally find the performance gap between methods to decrease, while ELF-P consistently outperforms baselines in all data regimes. See \OurAppendix \ref{app:robustness-datasize} for additional results.

\vspace{2mm}
\begin{wrapfigure}{rt}{44mm}
\vspace{-4mm}
\centering
\includegraphics[width=\linewidth]{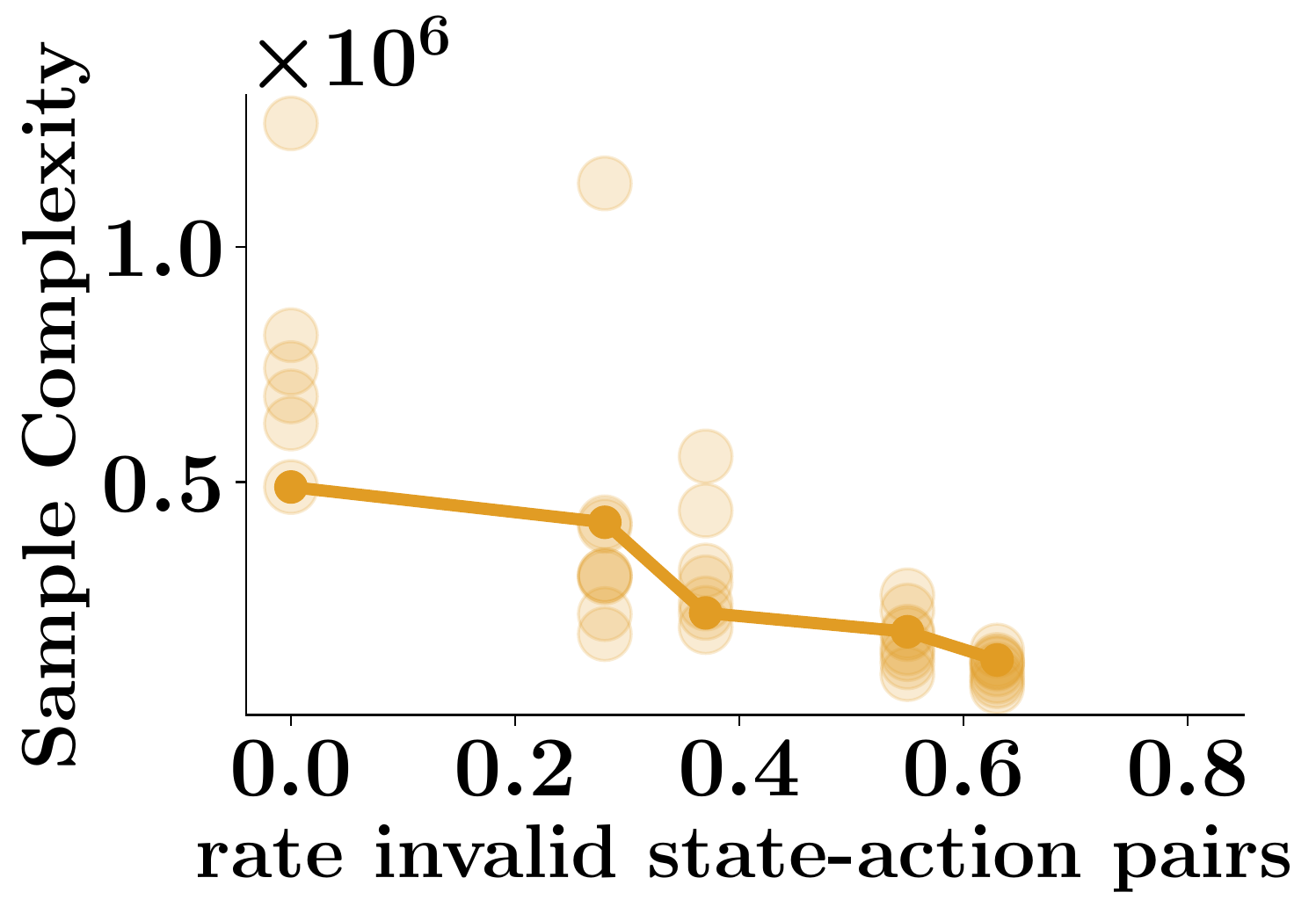}
\vspace{-6mm}
\caption{Effect of $\rho$ on sample complexity. Means across seeds are connected by lines.}
\vspace{-3mm}
\label{fig:sample-complxity}
\end{wrapfigure}
\subsubsection{Sample Complexity}
We further analyze the gains in sample complexity and compare them with the theoretical rates mentioned in Subsection \ref{ssec:learn_maskedQ1}, which hypothesize a linear dependency on the number of feasible action-state pairs in PAC-RL settings.
In particular, we measure the number of timesteps required to reach a success rate of 0.95 on (M) tasks with increasing values for the threshold $\rho$, thus pruning the state-action space more aggressively.
In Figure \ref{fig:sample-complxity} we plot this quantity against the ratio of infeasible state-action pairs for each value of $\rho$, computed according to the visitation distribution of a random policy.
While our algorithm recovers the performance of DDQN for $\rho=0$, we indeed observe a  linear trend in the decrease of sample complexity as the number of infeasible action-state pairs increases. Further details are in \OurAppendix \ref{app:sample-complexity}.

\subsubsection{Soft Prior Integration}
We implement  Soft-\alg, an ablation for our method which softens the integration of the prior into learning. This modification could potentially allow Soft-\alg to recover from degenerated priors, i.e., when play data has significant distribution mismatch with the target tasks of interest.
Soft-\alg samples from the prior (instead of sampling from the feasible set $\alpha$) both during initial exploration phase and while performing $\epsilon$-greedy exploration. Additionally, during exploitation, it multiplies the softmax of Q-values by the prior, thus biasing the greedy action selection towards the prior.
Finally, given the soft integration of the prior, it performs Q-learning over the set of all actions, instead of over the reduced set of feasible actions.
In Figure \ref{fig:soft_vs_hard} we compare the performance of \alg with Soft-\alg on the (H) task. Results on the (M) task are in the \OurAppendix \ref{app:soft-elfp}. We observe that while Soft-\alg is able to learn both medium and hard tasks, it has lower sample efficiency than \alg. The hard prior integration of \alg alleviates the Q-network from learning values for all state-action pairs: it can focus on learning Q-values for feasible state-action pairs only and ignore Q-values for unfeasible actions. This shows one of the main contributions of our algorithm.
While binarizing the prior results in greater learning efficiency, a soft integration could be  useful when dealing with degenerated priors and we reserve further exploration on the topic for future work.

\subsection{Real-world Experiments} \label{exp:spot}
In this experiment, we evaluate how \alg performs when transferred to physical hardware.
Figure \ref{fig:method} (upper right) depicts our physical robot workspace and further details can be found in the \OurAppendix \ref{app:hardware_exp}.
We use \textit{Boston Dynamics} Spot robot \cite{noauthor_spot_nodate} for all our experiments because of its stability and robustness when following desired end-effector trajectories.
The high-level planner, trained in simulation, is used at inference time to predict the required sequence of motion primitives to achieve a desired goal. We then instantiate each of them using established motion planning methods \cite{zimmermann2022dca} to create optimal control trajectories. For further implementation details, we refer the reader to \cite{zimmermann2021go}.

We show that \alg can be easily transferred to physical hardware while maintaining its goal-reaching capabilities.
In Figure \ref{fig:spot_priors} we show the agent executing several primitives in order to complete a task. We refer the reader to \OurAppendix \ref{app:hardware_exp} and the video material for extended visualizations.

\section{Discussion and Future Work} \label{sec:discussion}
\label{sec:conclusion}

We present \alg, a method that bridges motion planning and deep RL to achieve complex long-horizon manipulation tasks. 
We show that by integrating a discrete behavioral prior learned from easily collectable play data, we can achieve significant gains in sample efficiency compared to other baselines that leverage prior data.
This approach has the added benefit of largely avoiding infeasible actions during training.
By planning in a two-level hierarchy, we show how our method allows reasoning over long-horizons in a mixed decision space in an efficient manner. We finally demonstrate that within this framework, \alg can be easily transferred to physical hardware without further modifications, showing the potential of combining readily available motion planners with sample efficient RL algorithms. 

Despite showing promising results, our method assumes full observability of the state space and perfect execution of the motion primitives. These limitations could be addressed by introducing perception and by querying \alg at higher frequencies at inference time to ensure primitive completion.
 Furthermore, choosing a suitable level of abstractions for skills remains an open question, whose answer could relax the need for providing a predefined set of skills, while still maintaining a low-dimensional parametrization.
Future work may also actively learn the behavioral prior instead of leveraging a static play dataset, reaching a compromise between sample complexity and reliance on collected data.

We expect our work to enable future research directions such as a tighter coupling between the training of the high-level planner and the execution of motion primitives.
Although introducing motion planning in the training loop is time consuming, we believe that the significant gains in sample efficiency demonstrated in this work can help address this challenge.

\section*{ACKNOWLEDGMENT} \label{sec:acknowledg}

Núria Armengol and Marco Bagatella are financially supported by the Max Planck ETH Center for Learning Systems.
We thank Simon Zimmermann and Oliver Stark for the help throughout the project.





\newpage
\bibliographystyle{IEEEtran}
\bibliography{references}

\begin{thebibliography}{10}
\providecommand{\url}[1]{#1}
\csname url@samestyle\endcsname
\providecommand{\newblock}{\relax}
\providecommand{\bibinfo}[2]{#2}
\providecommand{\BIBentrySTDinterwordspacing}{\spaceskip=0pt\relax}
\providecommand{\BIBentryALTinterwordstretchfactor}{4}
\providecommand{\BIBentryALTinterwordspacing}{\spaceskip=\fontdimen2\font plus
\BIBentryALTinterwordstretchfactor\fontdimen3\font minus
  \fontdimen4\font\relax}
\providecommand{\BIBforeignlanguage}[2]{{%
\expandafter\ifx\csname l@#1\endcsname\relax
\typeout{** WARNING: IEEEtran.bst: No hyphenation pattern has been}%
\typeout{** loaded for the language `#1'. Using the pattern for}%
\typeout{** the default language instead.}%
\else
\language=\csname l@#1\endcsname
\fi
#2}}
\providecommand{\BIBdecl}{\relax}
\BIBdecl

\bibitem{winkler2018optimization}
A.~W. Winkler, ``Optimization-based motion planning for legged robots,'' Ph.D.
  dissertation, ETH Zurich, 2018.

\bibitem{winkler2018gait}
A.~W. Winkler, C.~D. Bellicoso, M.~Hutter, and J.~Buchli, ``Gait and trajectory
  optimization for legged systems through phase-based end-effector
  parameterization,'' \emph{IEEE Robotics and Automation Letters}, vol.~3,
  no.~3, pp. 1560--1567, 2018.

\bibitem{dai2014whole}
H.~Dai, A.~Valenzuela, and R.~Tedrake, ``Whole-body motion planning with
  centroidal dynamics and full kinematics,'' in \emph{2014 IEEE-RAS
  International Conference on Humanoid Robots}, 2014, pp. 295--302.

\bibitem{bjelonic2022offline}
M.~Bjelonic, R.~Grandia, M.~Geilinger, O.~Harley, V.~S. Medeiros, V.~Pajovic,
  E.~Jelavic, S.~Coros, and M.~Hutter, ``Offline motion libraries and online
  mpc for advanced mobility skills,'' \emph{The International Journal of
  Robotics Research}, 2022.

\bibitem{sermanet2021broadly}
P.~Sermanet, C.~Lynch \emph{et~al.}, ``Broadly-exploring, local-policy trees
  for long-horizon task planning,'' in \emph{5th Annual Conference on Robot
  Learning}, 2021.

\bibitem{eysenbach2019search}
B.~Eysenbach, R.~R. Salakhutdinov, and S.~Levine, ``Search on the replay
  buffer: Bridging planning and reinforcement learning,'' \emph{Advances in
  Neural Information Processing Systems}, vol.~32, 2019.

\bibitem{mainprice2020interior}
J.~Mainprice, N.~Ratliff, M.~Toussaint, and S.~Schaal, ``An interior point
  method solving motion planning problems with narrow passages,'' in \emph{2020
  29th IEEE International Conference on Robot and Human Interactive
  Communication (RO-MAN)}, 2020, pp. 547--552.

\bibitem{kalashnikov2021mt}
D.~Kalashnkov, J.~Varley, Y.~Chebotar, B.~Swanson, R.~Jonschkowski, C.~Finn,
  S.~Levine, and K.~Hausman, ``Mt-opt: Continuous multi-task robotic
  reinforcement learning at scale,'' \emph{arXiv preprint arXiv:2104.08212},
  2021.

\bibitem{akkaya2019solving}
I.~Akkaya, M.~Andrychowicz, M.~Chociej, M.~Litwin, B.~McGrew, A.~Petron,
  A.~Paino, M.~Plappert, G.~Powell, R.~Ribas \emph{et~al.}, ``Solving rubik's
  cube with a robot hand,'' \emph{arXiv preprint arXiv:1910.07113}, 2019.

\bibitem{andrychowicz2020learning}
O.~M. Andrychowicz, B.~Baker, M.~Chociej, R.~Jozefowicz, B.~McGrew,
  J.~Pachocki, A.~Petron, M.~Plappert, G.~Powell, A.~Ray \emph{et~al.},
  ``Learning dexterous in-hand manipulation,'' \emph{The International Journal
  of Robotics Research}, vol.~39, no.~1, pp. 3--20, 2020.

\bibitem{kalashnikov2018qt}
D.~Kalashnikov, A.~Irpan, P.~Pastor, J.~Ibarz, A.~Herzog, E.~Jang, D.~Quillen,
  E.~Holly, M.~Kalakrishnan, V.~Vanhoucke \emph{et~al.}, ``Qt-opt: Scalable
  deep reinforcement learning for vision-based robotic manipulation (2018),''
  \emph{arXiv preprint arXiv:1806.10293}, 2018.

\bibitem{nachum2018data}
O.~Nachum, S.~S. Gu, H.~Lee, and S.~Levine, ``Data-efficient hierarchical
  reinforcement learning,'' \emph{Advances in neural information processing
  systems}, vol.~31, 2018.

\bibitem{levy2018hierarchical}
A.~Levy, R.~Platt, and K.~Saenko, ``Hierarchical reinforcement learning with
  hindsight,'' \emph{arXiv preprint arXiv:1805.08180}, 2018.

\bibitem{coumans2021}
E.~Coumans and Y.~Bai, ``Pybullet, a python module for physics simulation for
  games, robotics and machine learning,'' \url{http://pybullet.org},
  2016--2021.

\bibitem{kingma2014}
D.~P. Kingma and J.~Ba, ``Adam: A method for stochastic optimization,''
  \emph{arXiv preprint arXiv:1412.6980}, 2014.

\bibitem{kavraki1996probabilistic}
L.~E. Kavraki, P.~Svestka, J.-C. Latombe, and M.~H. Overmars, ``Probabilistic
  roadmaps for path planning in high-dimensional configuration spaces,''
  \emph{IEEE transactions on Robotics and Automation}, vol.~12, no.~4, pp.
  566--580, 1996.

\bibitem{lavalle2001randomized}
S.~M. LaValle and J.~J. Kuffner~Jr, ``Randomized kinodynamic planning,''
  \emph{The international journal of robotics research}, vol.~20, no.~5, pp.
  378--400, 2001.

\bibitem{karaman2011sampling}
S.~Karaman and E.~Frazzoli, ``Sampling-based algorithms for optimal motion
  planning,'' \emph{The international journal of robotics research}, vol.~30,
  no.~7, pp. 846--894, 2011.

\bibitem{ratliff2009chomp}
N.~Ratliff, M.~Zucker, J.~A. Bagnell, and S.~Srinivasa, ``Chomp: Gradient
  optimization techniques for efficient motion planning,'' in \emph{2009 IEEE
  International Conference on Robotics and Automation}, 2009, pp. 489--494.

\bibitem{schulman2014motion}
J.~Schulman, Y.~Duan, J.~Ho, A.~Lee, I.~Awwal, H.~Bradlow, J.~Pan, S.~Patil,
  K.~Goldberg, and P.~Abbeel, ``Motion planning with sequential convex
  optimization and convex collision checking,'' \emph{The International Journal
  of Robotics Research}, vol.~33, no.~9, pp. 1251--1270, 2014.

\bibitem{kaelbling2013integrated}
L.~P. Kaelbling and T.~Lozano-P{\'e}rez, ``Integrated task and motion planning
  in belief space,'' \emph{The International Journal of Robotics Research},
  vol.~32, no. 9-10, pp. 1194--1227, 2013.

\bibitem{toussaint2015logic}
M.~Toussaint, ``Logic-geometric programming: An optimization-based approach to
  combined task and motion planning,'' in \emph{Twenty-Fourth International
  Joint Conference on Artificial Intelligence}, 2015.

\bibitem{toussaint2017multi}
M.~Toussaint and M.~Lopes, ``Multi-bound tree search for logic-geometric
  programming in cooperative manipulation domains,'' in \emph{2017 IEEE
  International Conference on Robotics and Automation (ICRA)}, 2017, pp.
  4044--4051.

\bibitem{garrett2020integrated}
C.~R. Garrett, R.~Chitnis, R.~Holladay, B.~Kim, T.~Silver, L.~P. Kaelbling, and
  T.~Lozano-P{\'e}rez, ``Integrated task and motion planning,'' \emph{Annual
  review of control, robotics, and autonomous systems}, vol.~4, pp. 265--293,
  2021.

\bibitem{braun2021rhh}
C.~V. Braun, J.~Ortiz-Haro, M.~Toussaint, and O.~S. Oguz, ``Rhh-lgp: Receding
  horizon and heuristics-based logic-geometric programming for task and motion
  planning,'' in \emph{2022 IEEE/RSJ International Conference on Intelligent
  Robots and Systems (IROS)}.\hskip 1em plus 0.5em minus 0.4em\relax IEEE,
  2022, pp. 13\,761--13\,768.

\bibitem{duan_one-shot_2017}
Y.~Duan, M.~Andrychowicz, B.~Stadie, O.~Jonathan~Ho, J.~Schneider,
  I.~Sutskever, P.~Abbeel, and W.~Zaremba, ``One-shot imitation learning,''
  \emph{Advances in neural information processing systems}, vol.~30, 2017.

\bibitem{barto2003recent}
A.~G. Barto and S.~Mahadevan, ``Recent advances in hierarchical reinforcement
  learning,'' \emph{Discrete event dynamic systems}, vol.~13, no.~1, pp.
  41--77, 2003.

\bibitem{dayan1992feudal}
P.~Dayan and G.~E. Hinton, ``Feudal reinforcement learning,'' \emph{Advances in
  neural information processing systems}, vol.~5, 1992.

\bibitem{wiering1997hq}
M.~Wiering and J.~Schmidhuber, ``Hq-learning,'' \emph{Adaptive Behavior},
  vol.~6, no.~2, pp. 219--246, 1997.

\bibitem{sutton1999between}
R.~S. Sutton, D.~Precup, and S.~Singh, ``Between mdps and semi-mdps: A
  framework for temporal abstraction in reinforcement learning,''
  \emph{Artificial intelligence}, vol. 112, no. 1-2, pp. 181--211, 1999.

\bibitem{dietterich2000hierarchical}
T.~G. Dietterich, ``Hierarchical reinforcement learning with the maxq value
  function decomposition,'' \emph{Journal of artificial intelligence research},
  vol.~13, pp. 227--303, 2000.

\bibitem{nasiriany2022augmenting}
S.~Nasiriany, H.~Liu, and Y.~Zhu, ``Augmenting reinforcement learning with
  behavior primitives for diverse manipulation tasks,'' in \emph{2022
  International Conference on Robotics and Automation (ICRA)}.\hskip 1em plus
  0.5em minus 0.4em\relax IEEE, 2022, pp. 7477--7484.

\bibitem{Kaelbling93learningto}
L.~P. Kaelbling, ``Learning to achieve goals,'' in \emph{IN PROC. OF IJCAI-93},
  1993, pp. 1094--1098.

\bibitem{andrychowicz2017hindsight}
M.~Andrychowicz, F.~Wolski, A.~Ray, J.~Schneider, R.~Fong, P.~Welinder,
  B.~McGrew, J.~Tobin, O.~Pieter~Abbeel, and W.~Zaremba, ``Hindsight experience
  replay,'' \emph{Advances in neural information processing systems}, vol.~30,
  2017.

\bibitem{gupta_relay_2019}
A.~Gupta, V.~Kumar, C.~Lynch, S.~Levine, and K.~Hausman, ``Relay policy
  learning: Solving long horizon tasks via imitation and reinforcement
  learning,'' \emph{Conference on Robot Learning (CoRL)}, 2019.

\bibitem{christen2020hide}
S.~{Christen}, L.~{Jendele}, E.~{Aksan}, and O.~{Hilliges}, ``Learning
  functionally decomposed hierarchies for continuous control tasks with path
  planning,'' \emph{IEEE Robotics and Automation Letters}, vol.~6, no.~2, pp.
  3623--3630, 2021.

\bibitem{ICLR16-hausknecht}
\BIBentryALTinterwordspacing
M.~Hausknecht and P.~Stone, ``Deep reinforcement learning in parameterized
  action space,'' in \emph{Proceedings of the International Conference on
  Learning Representations (ICLR)}, San Juan, Puerto Rico, May 2016. [Online].
  Available: \url{http://www.cs.utexas.edu/users/ai-lab?hausknecht:iclr16}
\BIBentrySTDinterwordspacing

\bibitem{dalal2021accelerating}
M.~Dalal, D.~Pathak, and R.~R. Salakhutdinov, ``Accelerating robotic
  reinforcement learning via parameterized action primitives,'' \emph{Advances
  in Neural Information Processing Systems}, vol.~34, pp. 21\,847--21\,859,
  2021.

\bibitem{masson2016reinforcement}
W.~Masson, P.~Ranchod, and G.~Konidaris, ``Reinforcement learning with
  parameterized actions,'' in \emph{Thirtieth AAAI Conference on Artificial
  Intelligence}, 2016.

\bibitem{relmogen2021}
F.~Xia, C.~Li, R.~Martín-Martín, O.~Litany, A.~Toshev, and S.~Savarese,
  ``Relmogen: Integrating motion generation in reinforcement learning for
  mobile manipulation,'' in \emph{2021 IEEE International Conference on
  Robotics and Automation (ICRA)}, 2021, pp. 4583--4590.

\bibitem{yamada21a}
J.~Yamada, Y.~Lee, G.~Salhotra, K.~Pertsch, M.~Pflueger, G.~Sukhatme, J.~Lim,
  and P.~Englert, ``Motion planner augmented reinforcement learning for robot
  manipulation in obstructed environments,'' in \emph{Proceedings of the 2020
  Conference on Robot Learning}, ser. Proceedings of Machine Learning Research,
  J.~Kober, F.~Ramos, and C.~Tomlin, Eds., vol. 155.\hskip 1em plus 0.5em minus
  0.4em\relax PMLR, 16--18 Nov 2021, pp. 589--603.

\bibitem{angelov2020}
D.~Angelov, Y.~Hristov, M.~Burke, and S.~Ramamoorthy, ``Composing diverse
  policies for temporally extended tasks,'' \emph{IEEE Robotics and Automation
  Letters}, vol.~5, no.~2, pp. 2658--2665, 2020.

\bibitem{vinyals2017new}
O.~Vinyals, T.~Ewalds, S.~Bartunov, P.~Georgiev, A.~Vezhnevets, M.~Yeo,
  A.~Makhzani, H.~K{\"u}ttler, J.~Agapiou, J.~Schrittwieser \emph{et~al.}, ``A
  new challenge for reinforcement learning,'' \emph{arXiv preprint
  ArXiv:1708.04782}, vol.~5, 2017.

\bibitem{berner2019dota}
C.~Berner, G.~Brockman, B.~Chan, V.~Cheung, P.~D{k{e}}biak, C.~Dennison,
  D.~Farhi, Q.~Fischer, S.~Hashme, C.~Hesse \emph{et~al.}, ``Dota 2 with large
  scale deep reinforcement learning,'' \emph{arXiv preprint arXiv:1912.06680},
  2019.

\bibitem{ye2020mastering}
D.~Ye, Z.~Liu, M.~Sun, B.~Shi, P.~Zhao, H.~Wu, H.~Yu, S.~Yang, X.~Wu, Q.~Guo
  \emph{et~al.}, ``Mastering complex control in moba games with deep
  reinforcement learning,'' in \emph{Proceedings of the AAAI Conference on
  Artificial Intelligence}, vol.~34, 2020, pp. 6672--6679.

\bibitem{boutilier2018planning}
C.~Boutilier, A.~Cohen, A.~Hassidim, Y.~Mansour, O.~Meshi, M.~Mladenov, and
  D.~Schuurmans, ``Planning and learning with stochastic action sets,'' in
  \emph{Proceedings of the Twenty-Seventh International Joint Conference on
  Artificial Intelligence, {IJCAI-18}}, 2018, pp. 4674--4682.

\bibitem{zahavy2018learn}
T.~Zahavy, M.~Haroush, N.~Merlis, D.~J. Mankowitz, and S.~Mannor, ``Learn what
  not to learn: Action elimination with deep reinforcement learning,''
  \emph{Advances in Neural Information Processing Systems}, vol.~31, 2018.

\bibitem{even2003action}
E.~Even-Dar, S.~Mannor, and Y.~Mansour, ``Action elimination and stopping
  conditions for reinforcement learning,'' in \emph{Proceedings of the 20th
  International Conference on Machine Learning (ICML-03)}, 2003, pp. 162--169.

\bibitem{gibson1977theory}
J.~J. Gibson, ``The theory of affordances,'' \emph{Hilldale, USA}, vol.~1,
  no.~2, pp. 67--82, 1977.

\bibitem{khetarpal2020what}
K.~Khetarpal, Z.~Ahmed, G.~Comanici, D.~Abel, and D.~Precup, ``What can i do
  here? a theory of affordances in reinforcement learning,'' in \emph{ICML},
  2020.

\bibitem{costales2022possibility}
R.~Costales, S.~Iqbal, and F.~Sha, ``Possibility before utility: Learning and
  using hierarchical affordances,'' in \emph{International Conference on
  Learning Representations}, 2022.

\bibitem{pertsch2020spirl}
K.~Pertsch, Y.~Lee, and J.~J. Lim, ``Accelerating reinforcement learning with
  learned skill priors,'' in \emph{Conference on Robot Learning (CoRL)}, 2020.

\bibitem{pertsch2021skild}
K.~Pertsch, Y.~Lee, Y.~Wu, and J.~J. Lim, ``Demonstration-guided reinforcement
  learning with learned skills,'' \emph{5th Conference on Robot Learning},
  2021.

\bibitem{singh2021parrot}
A.~Singh, H.~Liu, G.~Zhou, A.~Yu, N.~Rhinehart, and S.~Levine, ``Parrot:
  Data-driven behavioral priors for reinforcement learning,'' in
  \emph{International Conference on Learning Representations}, 2021.

\bibitem{tirumala2022behavior}
D.~Tirumala, A.~Galashov, H.~Noh, L.~Hasenclever, R.~Pascanu, J.~Schwarz,
  G.~Desjardins, W.~M. Czarnecki, A.~Ahuja, Y.~W. Teh \emph{et~al.}, ``Behavior
  priors for efficient reinforcement learning,'' \emph{Journal of Machine
  Learning Research}, vol.~23, no. 221, pp. 1--68, 2022.

\bibitem{lynch_learning_2019}
C.~Lynch, M.~Khansari, T.~Xiao, V.~Kumar, J.~Tompson, S.~Levine, and
  P.~Sermanet, ``Learning latent plans from play,'' \emph{Conference on Robot
  Learning (CoRL)}, 2019.

\bibitem{noauthor_spot_nodate}
``Spot® {\textbar} {Boston} {Dynamics},''
  \url{https://www.bostondynamics.com/products/spot}, 2022.

\bibitem{saycan2022arxiv}
M.~Ahn, A.~Brohan, N.~Brown, Y.~Chebotar, O.~Cortes, B.~David, C.~Finn, C.~Fu,
  K.~Gopalakrishnan, K.~Hausman, A.~Herzog, D.~Ho, J.~Hsu, J.~Ibarz, B.~Ichter,
  A.~Irpan, E.~Jang, R.~J. Ruano, K.~Jeffrey, S.~Jesmonth, N.~Joshi, R.~Julian,
  D.~Kalashnikov, Y.~Kuang, K.-H. Lee, S.~Levine, Y.~Lu, L.~Luu, C.~Parada,
  P.~Pastor, J.~Quiambao, K.~Rao, J.~Rettinghouse, D.~Reyes, P.~Sermanet,
  N.~Sievers, C.~Tan, A.~Toshev, V.~Vanhoucke, F.~Xia, T.~Xiao, P.~Xu, S.~Xu,
  M.~Yan, and A.~Zeng, ``Do as i can and not as i say: Grounding language in
  robotic affordances,'' in \emph{arXiv preprint arXiv:2204.01691}, 2022.

\bibitem{zhu2022bottom}
Y.~Zhu, P.~Stone, and Y.~Zhu, ``Bottom-up skill discovery from unsegmented
  demonstrations for long-horizon robot manipulation,'' \emph{IEEE Robotics and
  Automation Letters}, 2022.

\bibitem{konidaris2010constructing}
G.~Konidaris, S.~Kuindersma, R.~Grupen, and A.~Barto, ``Constructing skill
  trees for reinforcement learning agents from demonstration trajectories,''
  \emph{Advances in neural information processing systems}, vol.~23, 2010.

\bibitem{niekum2015online}
S.~Niekum, S.~Osentoski, C.~G. Atkeson, and A.~G. Barto, ``Online bayesian
  changepoint detection for articulated motion models,'' in \emph{2015 IEEE
  International Conference on Robotics and Automation (ICRA)}, 2015, pp.
  1468--1475.

\bibitem{schaul2015universal}
T.~Schaul, D.~Horgan, K.~Gregor, and D.~Silver, ``Universal value function
  approximators,'' in \emph{International conference on machine learning},
  2015, pp. 1312--1320.

\bibitem{puterman1994}
M.~L. Puterman, \emph{Markov Decision Processes}.\hskip 1em plus 0.5em minus
  0.4em\relax Wiley, 1994.

\bibitem{kakade2003sample}
S.~M. Kakade, \emph{On the sample complexity of reinforcement learning}.\hskip
  1em plus 0.5em minus 0.4em\relax University of London, University College
  London (United Kingdom), 2003.

\bibitem{lattimore2012pac}
T.~Lattimore and M.~Hutter, ``Pac bounds for discounted mdps,'' in
  \emph{International Conference on Algorithmic Learning Theory}, 2012, pp.
  320--334.

\bibitem{strehl2006pac}
A.~L. Strehl, L.~Li, E.~Wiewiora, J.~Langford, and M.~L. Littman, ``Pac
  model-free reinforcement learning,'' in \emph{Proceedings of the 23rd
  international conference on Machine learning}, 2006, pp. 881--888.

\bibitem{bertsekas1996neuro}
D.~P. Bertsekas and J.~N. Tsitsiklis, \emph{Neuro-dynamic programming}, ser.
  Optimization and neural computation series.\hskip 1em plus 0.5em minus
  0.4em\relax Athena Scientific, 1996, vol.~3.

\bibitem{watkins1992q}
C.~J. Watkins and P.~Dayan, ``Q-learning,'' \emph{Machine learning}, vol.~8,
  no.~3, pp. 279--292, 1992.

\bibitem{van2016deep}
H.~Van~Hasselt, A.~Guez, and D.~Silver, ``Deep reinforcement learning with
  double q-learning,'' in \emph{Proceedings of the AAAI conference on
  artificial intelligence}, vol.~30, 2016.

\bibitem{mnih2015human}
V.~Mnih, K.~Kavukcuoglu, D.~Silver, A.~A. Rusu, J.~Veness, M.~G. Bellemare,
  A.~Graves, M.~Riedmiller, A.~K. Fidjeland, G.~Ostrovski \emph{et~al.},
  ``Human-level control through deep reinforcement learning,'' \emph{Nature},
  vol. 518, no. 7540, pp. 529--533, 2015.

\bibitem{fujimoto2018addressing}
S.~Fujimoto, H.~Hoof, and D.~Meger, ``Addressing function approximation error
  in actor-critic methods,'' in \emph{International conference on machine
  learning}, 2018, pp. 1587--1596.

\bibitem{lin1992reinforcement}
L.-J. Lin, \emph{Reinforcement learning for robots using neural
  networks}.\hskip 1em plus 0.5em minus 0.4em\relax Carnegie Mellon University,
  1992.

\bibitem{hester2018deep}
T.~Hester, M.~Vecerik, O.~Pietquin, M.~Lanctot, T.~Schaul, B.~Piot, D.~Horgan,
  J.~Quan, A.~Sendonaris, I.~Osband \emph{et~al.}, ``Deep q-learning from
  demonstrations,'' in \emph{Proceedings of the AAAI Conference on Artificial
  Intelligence}, vol.~32, 2018.

\bibitem{zimmermann2022dca}
S.~Zimmermann, M.~Busenhart, S.~Huber, R.~Poranne, and S.~Coros,
  ``Differentiable collision avoidance using collision primitives,'' in
  \emph{2022 IEEE/RSJ International Conference on Intelligent Robots and
  Systems (IROS)}.\hskip 1em plus 0.5em minus 0.4em\relax IEEE, 2022, pp.
  8086--8093.

\bibitem{zimmermann2021go}
S.~Zimmermann, R.~Poranne, and S.~Coros, ``Go fetch!-dynamic grasps using
  boston dynamics spot with external robotic arm,'' in \emph{2021 IEEE
  International Conference on Robotics and Automation (ICRA)}, 2021, pp.
  4488--4494.

\end{thebibliography}
\newpage
\renewcommand{\thetable}{S\arabic{table}}
\renewcommand{\thefigure}{S\arabic{figure}}
\renewcommand{\theequation}{S\arabic{equation}}
\section*{Appendix} \label{sec:appendix}
\subsection{Environment}
\label{app:env}

\subsubsection{Training environment}

The environment is built on \textit{Pybullet} physics simulation software \cite{coumans2021} and it is shown in Figure \ref{fig:pybullet-env}.
The states $\in \mathbb{R}^{11}$ include the 3D robot's end-effector position, a binary variable representing the gripper state (open/close), the 3D position of the block and the 1D joint position for each of the 3 drawers and the door. The goal state  $\in \mathcal{R}^3$ is the desired block position (e.g. behind the door, inside the mid-drawer or somewhere on the table).
The action space is discrete and consists of 10 object-centric motion primitives, namely reaching every object, manipulating the objects (e.g. sliding or pulling), and opening and closing the gripper. Given the primitives we use are relational to objects they become implicitly parameterized by the corresponding object pose. We also include a \textit{center} action to move the robot back to the center of the desk to have greater reachability. Finally, we also include a \textit{go-to-goal} primitive that moves the end-effector over the goal position. A complete action list is shown below:
\begin{multicols}{2}
\begin{enumerate}
    \item Go to door handle
    \item Go to drawer1 handle
    \item Go to drawer2 handle
    \item Go to drawer3 handle
    \item Go to center
    \item Go to block
    \item Go to goal
    \item Grasp/release (Open/close the gripper)
    \item Pull/push
    \item Slide left/slide right
\end{enumerate}
\end{multicols}

\begin{wrapfigure}{r}{45mm}
\centering
\includegraphics[width=\linewidth]{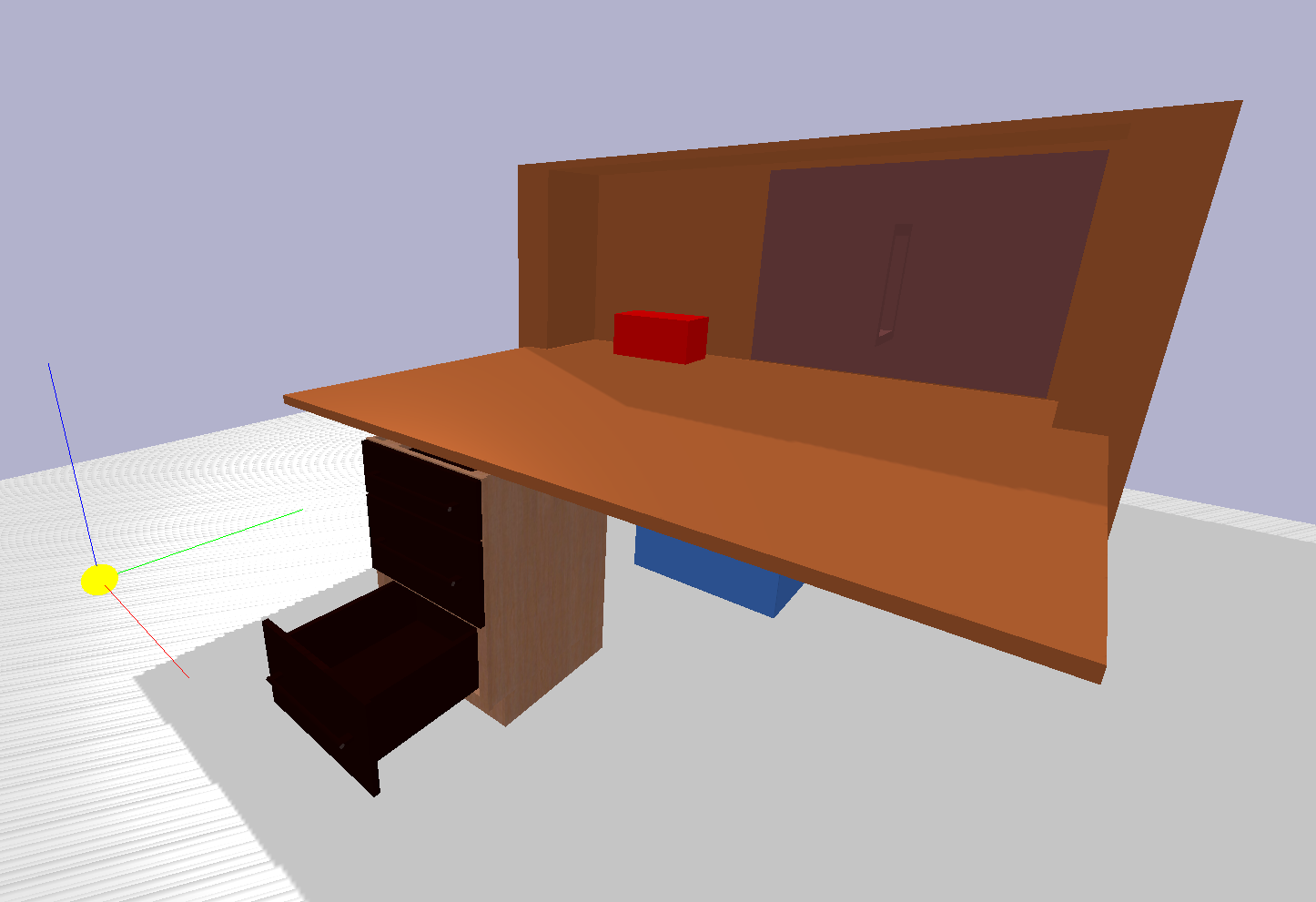}
\caption{Training desk environment with different objects that can be manipulated. The robot end-effector is represented with a yellow sphere.}
\label{fig:pybullet-env}
\end{wrapfigure}

In the case of multiple primitives for one action (i.e. actions 8, 9 and 10), the executed primitive depends on the current state (e.g. the gripper opens if its current state is closed and viceversa).

If the agent attempts an infeasible action, such as manipulating an object with the wrong primitive (e.g. pulling a sliding door) or moving in a non collision-free path, the environment does not perform a simulation step and the state remains unaltered.  To check for collisions, we use the \code{rayTest} method in \textit{Pybullet}, which performs a single raycast.

We use the sparse reward signal: $R(\cdot|s, g) = \mathbbm{1}_{\{|f(s)-g|_1 < \epsilon\}}$, where $f(s)$ extracts the current block position from the state and $\epsilon=0.1$ is a threshold determining task success.

\subsection{Extended Experimental Results} 
\label{app:experiments}
\subsubsection{HER}

We experiment with combining our method with off-policy goal relabeling \cite{andrychowicz2017hindsight}.
We evaluate different relabeling ratios $k$. As already mentioned in the Results, we observe that the greater the relabeling ratio, the slower the convergence to the optimal policy (see Figure \ref{fig:her-results}). We hypothesize that this is because the environment dynamics are not smooth and the policy fails to generalize to distant goals despite the relabeling, which may hurt performance.

\begin{figure}
\begin{center}
\vspace{2mm}

\minipage{\linewidth}
\small
\centering
\textcolor{elfp}{\rule[2pt]{15pt}{3pt}} \textrm{ELF-P}\quad
\textcolor{her}{\rule[2pt]{15pt}{3pt}} \textrm{ELF-P+HER(K=2)}\quad
\textcolor{ddqn}{\rule[2pt]{15pt}{3pt}} \textrm{ELF-P+HER(K=4)}
\endminipage

\end{center}
\minipage{\linewidth}
\includegraphics[width=0.49\textwidth]{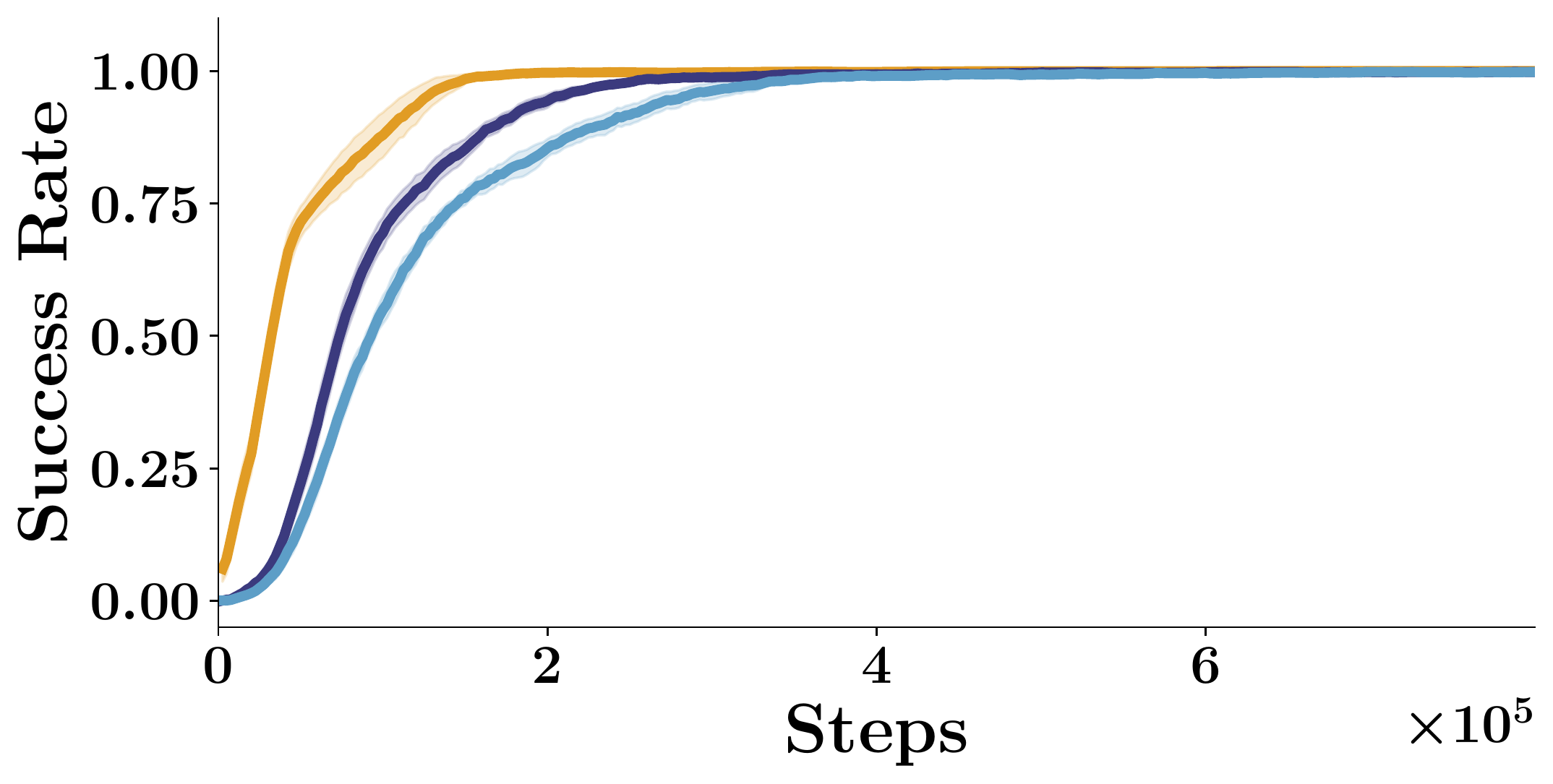}
\label{fig:her-effect-medium}
\includegraphics[width=0.49\textwidth]{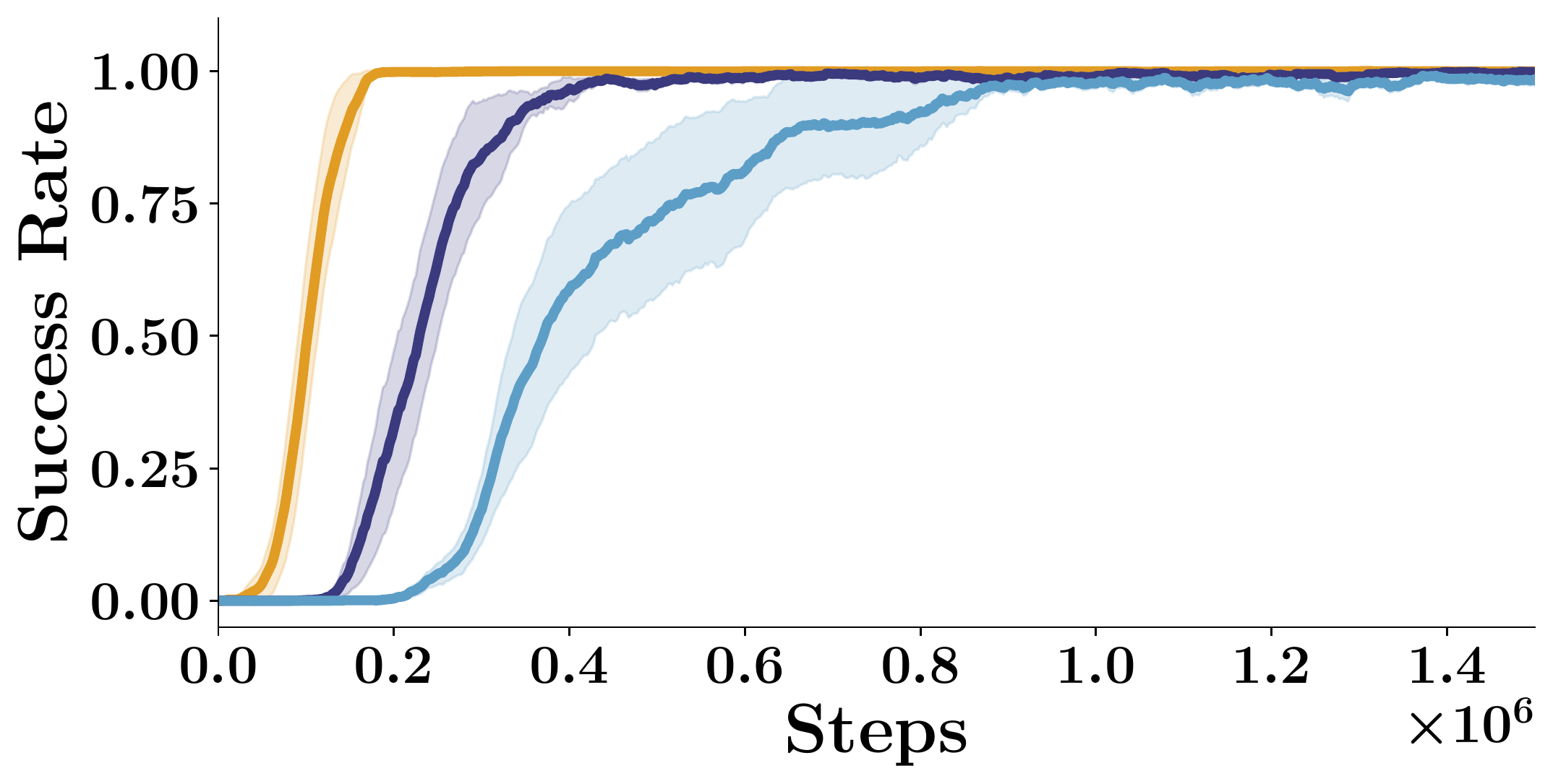}
\label{fig:her-effect-hard}
\endminipage
\vspace{-5mm}
\caption{Effect of relabeling ratio on success rate for the Medium (left) and Hard (right) task variants. Results are averaged across 10 independent random seeds, shaded area represents standard deviation. (Play dataset size $10^5$). }
\label{fig:her-results}
\vspace{2mm}
\end{figure}

\subsubsection{Sample complexity}
\label{app:sample-complexity}
Results reported in subsection \textit{Sample Complexity} in Section \ref{sec:Experiments} are obtained by running 5 independent random seeds for each value of $\rho$ in the Medium task.
We report results obtained for the Hard task in Figure \ref{fig:sample-complxity-hard}, which shows an increase of sample complexity as the number of feasible actions increases.
However, in this hard setting, several seeds for several values of $\rho$ don't reach a success rate of 0.95. This is expected given that for example $\rho=0$ recovers DDQN behavior which was shown to fail in learning the Hard task. Consequently, in order not to bias the results, we opt for only reporting the results on $\rho$ values whose seeds are always successful in learning.

\begin{wrapfigure}{r}{45mm}
\centering
\includegraphics[width=\linewidth]{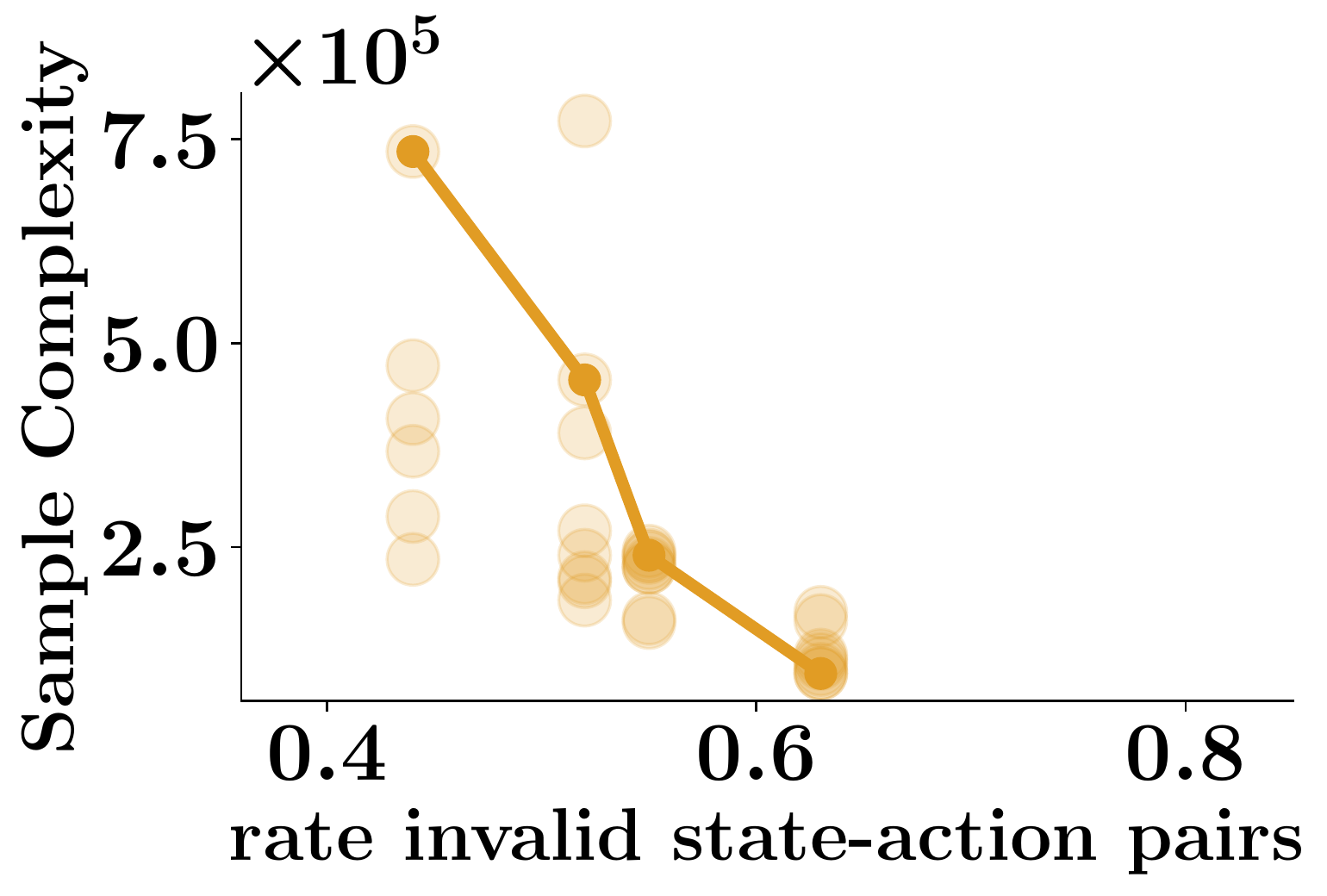}
\caption{Effect of $\rho$ on sample complexity for the Hard task. Means across seeds are connected by lines.}
\label{fig:sample-complxity-hard}
\end{wrapfigure}

\subsubsection{Robustness to play dataset size}
\label{app:robustness-datasize}
We study the effect of play dataset size on training performance for both \alg and DDQN+Prefill (which affects the amount of data used for training the prior for \alg and the amount of data used to prefill the replay buffer for DDQN+Prefill). In Figure \ref{fig:trainig-curves-datasize-effect} we show the resulting training curves using 3 different dataset sizes for \alg and DDQN+Prefill.

\begin{figure*}[!htb]
\begin{center}

\minipage{\textwidth}
\small
\centering
\textcolor{elfp}{\rule[2pt]{20pt}{3pt}} \textrm{ELF-P} \quad
\
\textcolor{prefill}{\rule[2pt]{20pt}{3pt}} \textrm{DDQN+Prefill} \quad

\endminipage
\vspace{-1mm}
\end{center}
\minipage{\linewidth}
\includegraphics[width=0.24\textwidth]{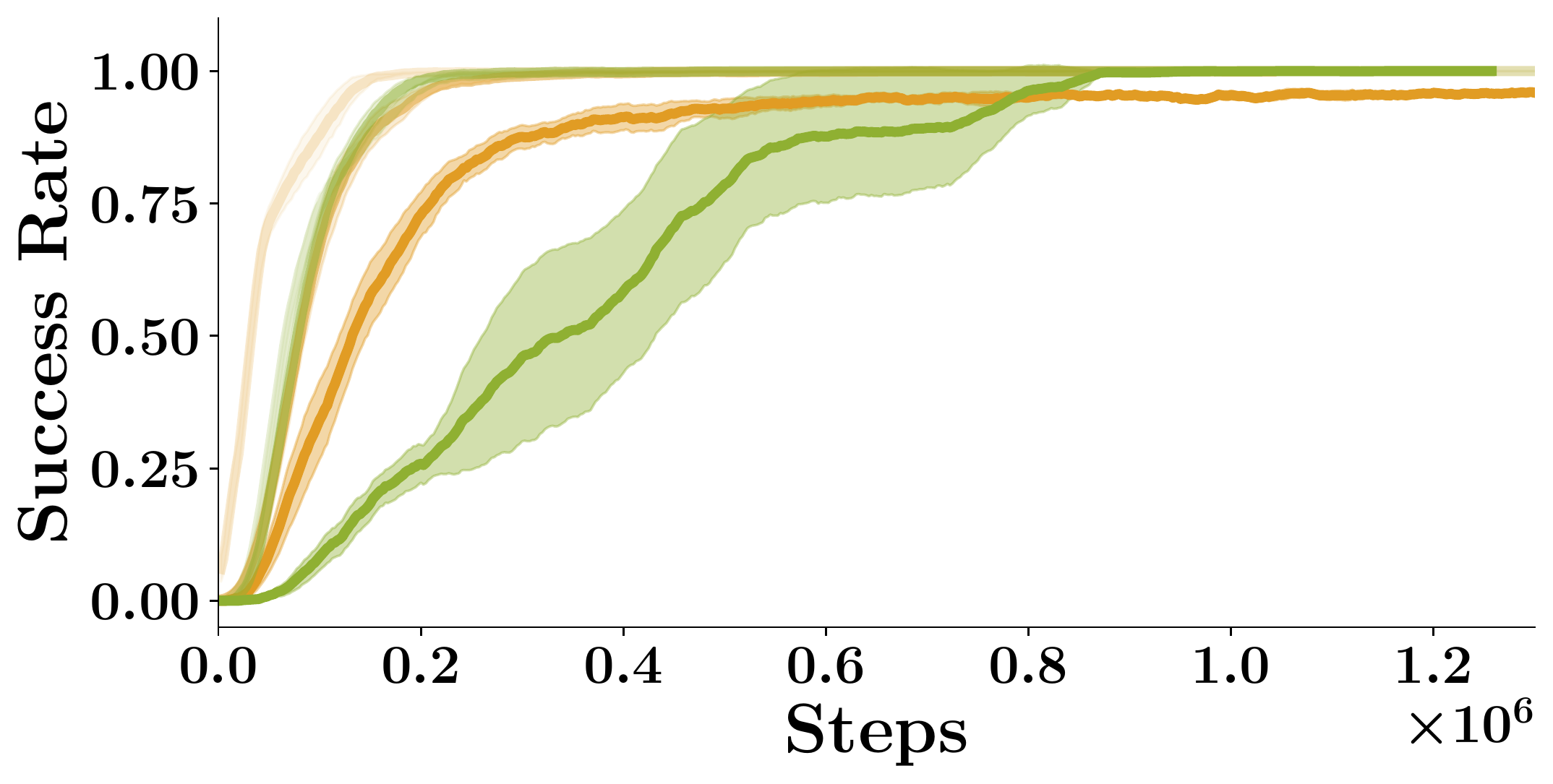}
\label{fig:medium_datasize_success}
\includegraphics[width=0.24\textwidth]{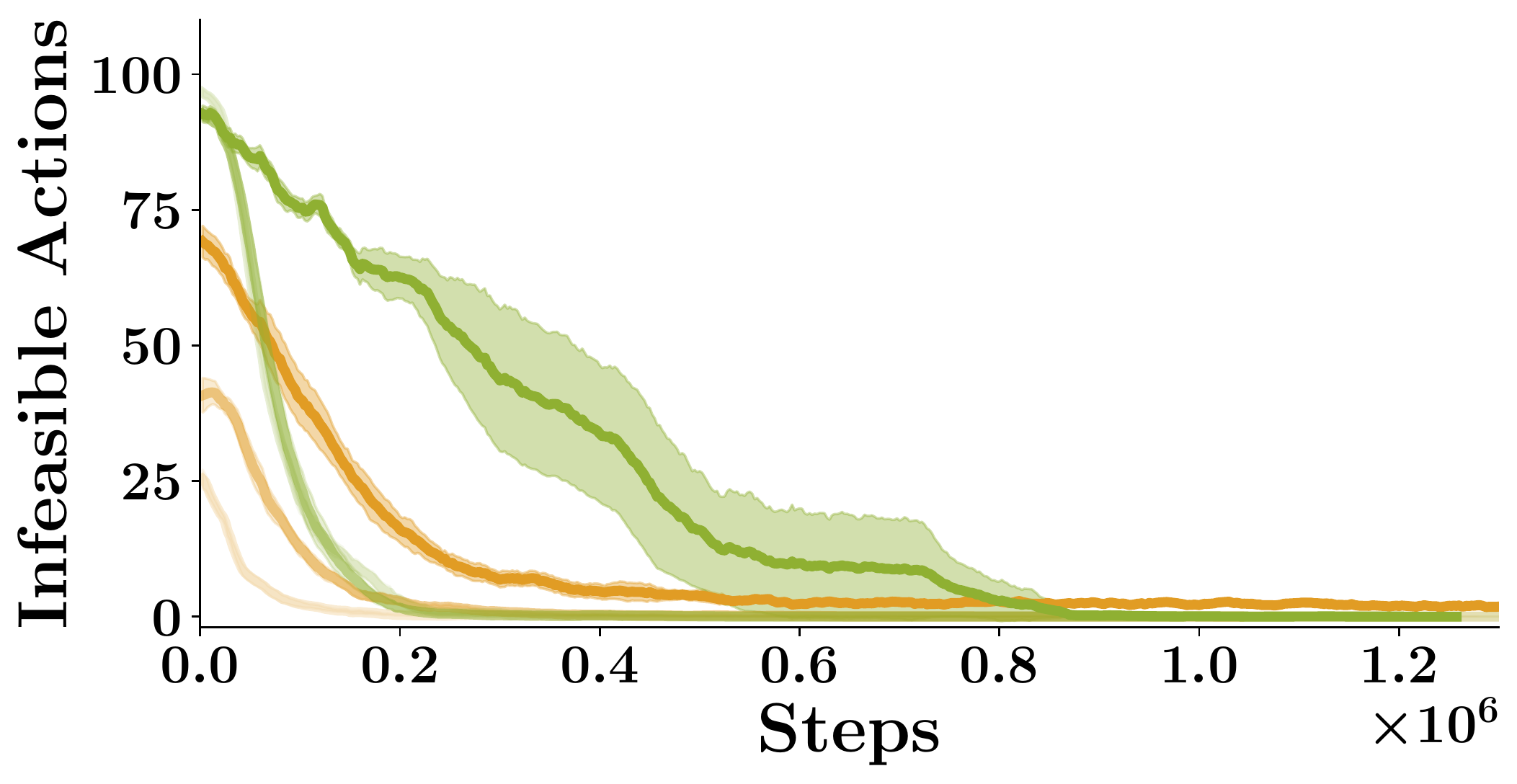}
\label{fig:medium_datasize_inv}
\includegraphics[width=0.24\textwidth]{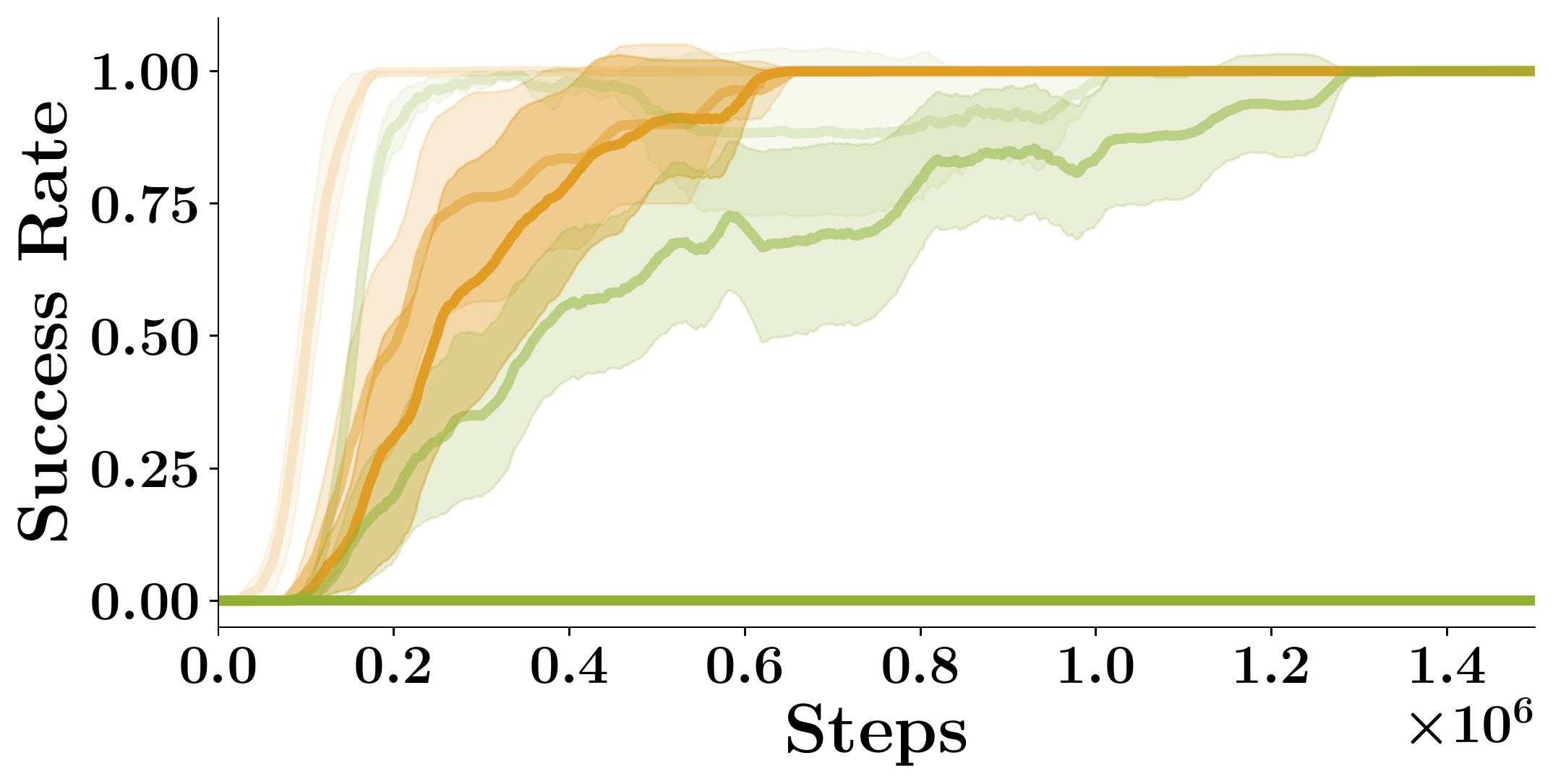}
\label{fig:hard_datasize_success}
\includegraphics[width=0.24\textwidth]{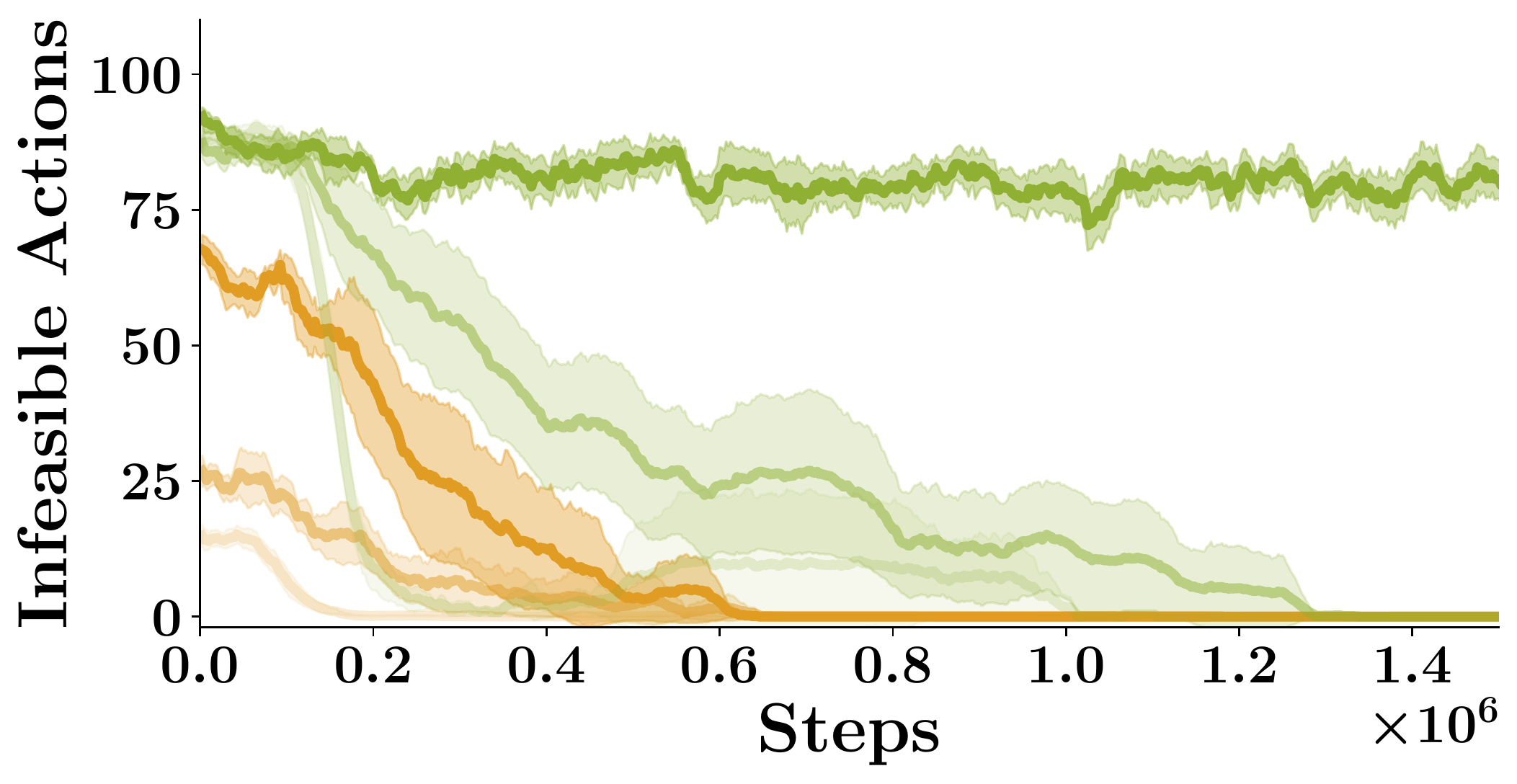}
\label{fig:hard_datasize_inv}
\endminipage
\vspace{-5mm}
\caption{Success rate and number of infeasible actions attempts for the Medium (top) and Hard
(bottom) task variants for 3 different dataset sizes: from darker to lighter $10^3, 10^4 \text{ and } 10^5$ datapoints respectively. When using $10^3$ datapoints DDQN+Prefill fails completely in solving the Hard task.(Results are averaged
across 10 independent random seeds, shaded area represents standard deviation).}
\vspace{1mm}
\label{fig:trainig-curves-datasize-effect}
\end{figure*}

\subsubsection{Soft prior integration}
\label{app:soft-elfp}
We implement  Soft-\alg, an ablation for our method which softens the integration of the prior into learning. This modification could potentially allow Soft-\alg to recover from degenerated priors, i.e., when play data has significant distribution mismatch with the target tasks of interest.
Soft-\alg samples from the prior (instead of sampling from the feasible set $\alpha$) both during initial exploration phase and while performing $\epsilon$-greedy exploration. Additionally, during exploitation, it multiplies the softmax of Q-values by the prior, thus biasing the greedy action selection towards the prior.
Finally, given the soft integration of the prior, it performs Q-learning over the set of all actions, instead of over the reduced set of feasible actions.
In Figure \ref{fig:soft_vs_hard_medium} we compare the performance of \alg with Soft-\alg. We observe that while Soft-\alg is able to learn both medium and hard tasks, it has lower sample efficiency than \alg. One of the main reasons of \alg being faster is because a hard prior integration alleviates the Q-network from learning values for all state-action pairs: it can focus on learning Q-values for feasible state-action pairs only and ignore Q-values for unfeasible actions. This shows one of the main contributions of our algorithm.
Nevertheless, a soft integration could be useful when dealing with degenerated priors and we reserve further exploration on the topic for future work.

\begin{figure}
\begin{center}
\vspace{2mm}

\minipage{\linewidth}
\small
\centering
\textcolor{elfp}{\rule[2pt]{20pt}{3pt}} \textrm{ELF-P} \quad
\textcolor{soft-elfp}{\rule[2pt]{20pt}{3pt}} \textrm{SOFT ELF-P}
\endminipage

\vspace{-1mm}
\end{center}
\minipage{\linewidth}
\includegraphics[width=0.49\textwidth]{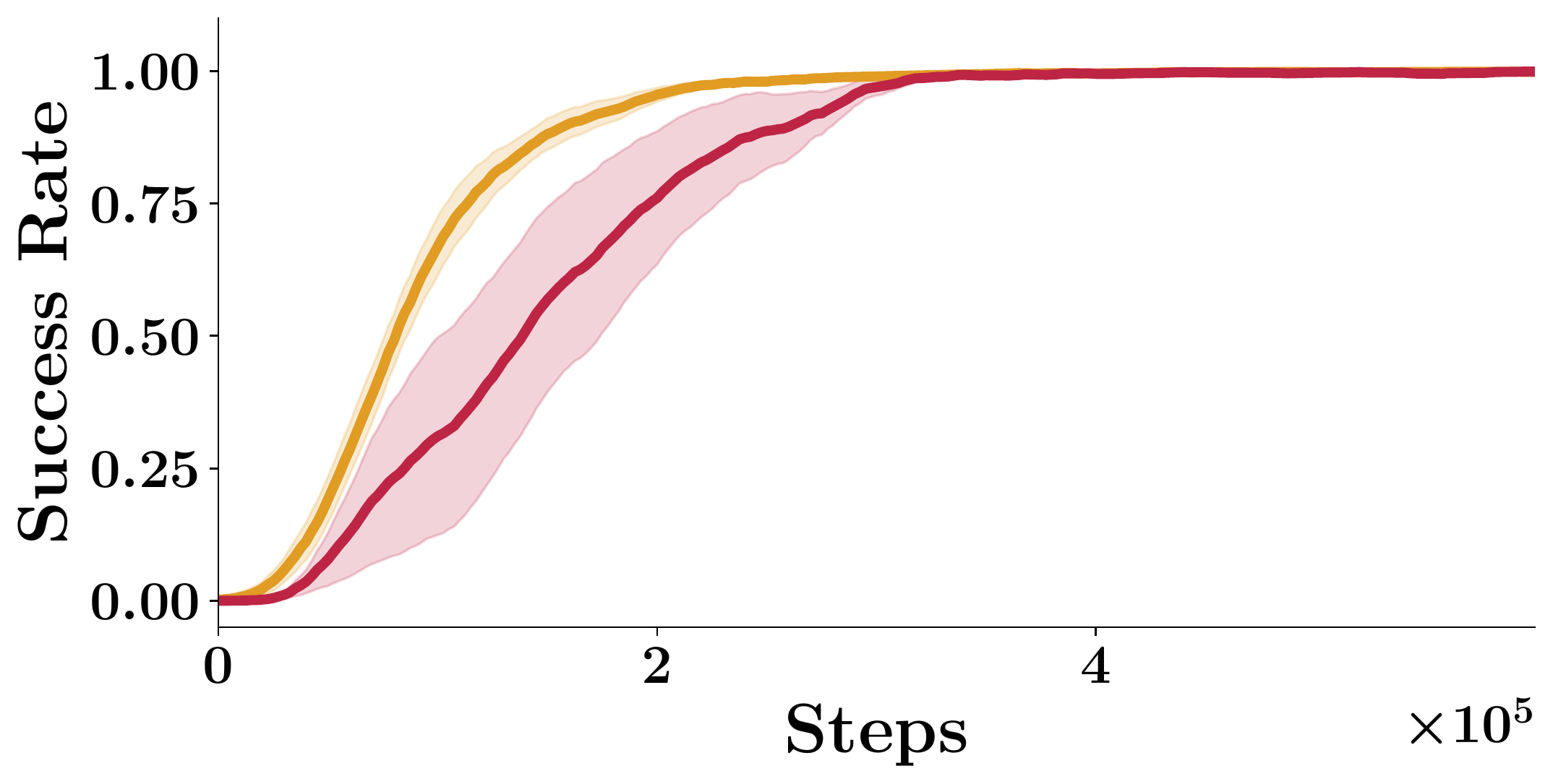}
\label{fig:medium_soft_success}
\includegraphics[width=0.49\textwidth]{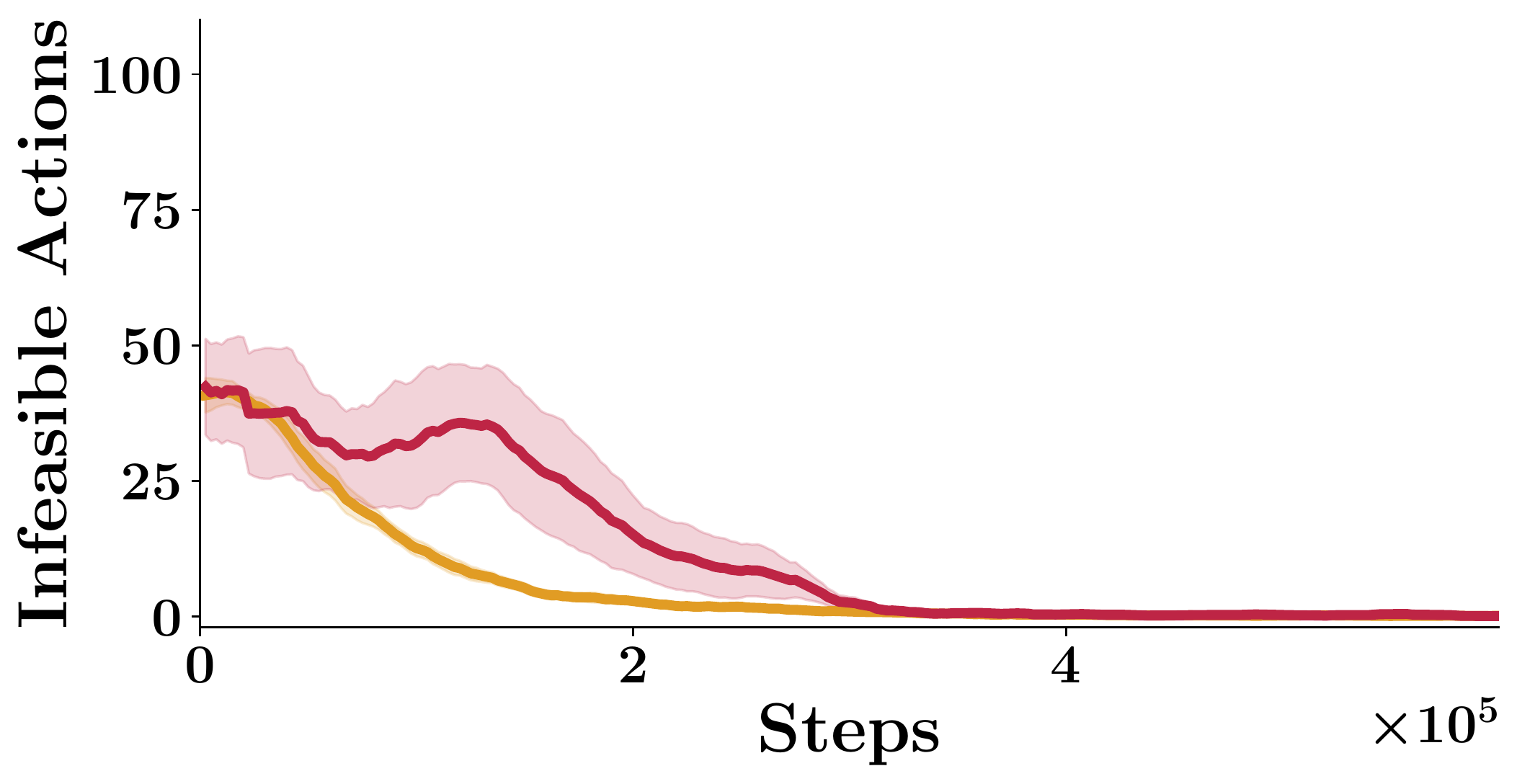}
\label{fig:medium_soft_inv}
\endminipage
\vspace{-5mm}
\caption{Comparison between \alg and Soft-\alg. \alg. Success rate and number of infeasible actions attempts for the Hard  task variant. (Play dataset size $10^4$). Results are averaged across 10 independent
random seeds, shaded area represents standard deviation) }
\label{fig:soft_vs_hard_medium}
\vspace{4mm}
\end{figure}

\subsection{Additional experimental details}
\label{app:additional_exp_details}

\subsubsection{Evaluation protocol} 
All reported results are obtained by evaluating the policies over 50 episodes every 2500 environment steps. Results are averaged over 10 random independent seeds unless explicitly stated. Shaded areas represent standard deviation across seeds.

\subsubsection{\textit{Play}-dataset collection}
Given that our method can learn with very little data ($\sim$1h30min of data collection), play data could in practice be collected by a few human operators choosing from a predefined set of primitives.
However for simplicity, we resort to evaluating whether a termination condition for each primitive is met in every configuration. This is a common approach under the options framework \cite{sutton1999between}. 
The tuples $(s,a)$, containing the state of the environment and the feasible action performed in the state respectively, are stored in the \textit{play}-dataset $\mathcal{D}$ for training the prior.

\subsubsection{Hyperparameters and Architectures}
\label{app:baselines}
Across all baselines and experiments, unless explicitly stated, we use the same set of hyperparameters and neural network architectures Values are reported in Table \ref{table:hyperparams-new}. We choose a set of values that was found to perform well in previous works. We report the list of hyperparameters that were tuned for each method in the following subsections. 

\begin{table}[ht]
\small
\centering
\begin{adjustbox}{width=0.5\textwidth}
\begin{tabular}{ccc}
    \toprule
    \textbf{Parameter}  & \textbf{Value}\\ 
    \midrule
    Q-network architecture  & MLP [128, 256]\\
    Batch size       & 256\\
    Exploration technique  & $\epsilon$-greedy with exponential decay\\
    Initial $\epsilon$ & 0.5 \\
    Decay rate for $\epsilon$  & 5e-5 \\
    Discount $\gamma$  & 0.97 \\
    Optimizer & Adam ($\beta_1=0.9, \beta_2=0.999$) \cite{kingma2014}\\
    Learning rate $\eta$ & 1e-4 \\
    Episode length $T$ & 100 \\
    Experience replay size & 1e6 \\
    Initial exploration steps & 2000 \\
    Steps before training starts & 1000 \\
    Steps between parameter updates & 50 \\
    Soft target update parameter $\mu$ & 0.995 \\
    Threshold $\rho$ & 0.01 \\
   \bottomrule
\end{tabular}
\end{adjustbox}
\caption{Architecture parameters and hyperparameters used for all the baselines and \alg.}
\label{table:hyperparams-new}
\end{table}

\begin{table}[ht]
\centering
\small
\label{table:hyperparams-prior}
\begin{tabular}{ccc}
    \toprule
    \textbf{Parameter}  & \textbf{Value} \\ 
    \midrule
    Prior architecture & MLP [200, 200] \\
    Prior batch size & 500 \\    
    Prior training steps & 1e5 \\
    Dataset size & 100000 \\
    Prior optimizer & Adam($\beta_1=0.9, \beta_2=0.999$ \cite{kingma2014} \\
    Learning rate & 1e-3 \\
   \bottomrule
\end{tabular}
\caption{Architecture parameters and hyperparameters used for training the prior in \alg.}
\end{table}

\paragraph{DDQN}
      \begin{itemize}
        \item Discount $\gamma$: $0.97$. Tuned over $[0.90, 0.95, 0.97, 0.99]$. 
      \end{itemize}

\paragraph{DDQN+HER}
      \begin{itemize}
        \item Relabeling ratio $k$: $4$. As suggested in the original paper \cite{andrychowicz2017hindsight}.
        \item Discount $\gamma$: $0.97$. Tuned over $[0.90, 0.95, 0.97, 0.99]$. 
      \end{itemize}

\paragraph{DDQN+Prefill}
\begin{itemize}
    \item Discount $\gamma$: $0.95$. Tuned over $[0.90, 0.95, 0.97, 0.99]$. 
     \item Prefill dataset: 
 The dataset used to prefill the replay buffer is the same as the \textit{Play}-dataset used to train the behavioral prior, but extending the tuples $(s,a) \to (s,a,s',g,r)$ to be able to perform TD-loss on them. 
Given that the \textit{Play}-dataset is task-agnostic, we decide to compute the rewards with relabelling, i.e., to relabel the goal $g$ to the achieved goal and to set the reward $r$ to $1$, otherwise set the reward to zero. We experiment with different relabeling frequencies $f \in [0.0, 0.25, 0.5, 0.75, 1.0]$.
We find that $f>0$ leads to overestimation of the Q-values at early stages of training and thus hurt performance. For this reason we use $f=0$, i.e., no relabelling.
 \end{itemize} 
 
\paragraph{DQfD}
\begin{itemize}
    \item Discount $\gamma$: $0.95$. Tuned over $[0.90, 0.95, 0.97, 0.99]$.
    \item Large margin classification loss weight $\lambda_2$: $1e^{-3}$. Tuned over $[1e^{-2}, 1e^{-3}, 1e^{-4}, 1e^{-5}]$.
    \item Expert margin: $0.05$. Tuned over $[0.01, 0.05, 0.1, 0.5, 0.8]$.
    \item L2 regularization loss weight $\lambda_3$: $1e^{-5}$.
    \item Prefill dataset: Same approach as for DDQN+Prefill explained above.
\end{itemize} 

\paragraph{SOFT ELF-P}
\begin{itemize}
    \item Discount $\gamma$: $0.97$. Tuned over $[0.90, 0.95, 0.97, 0.99]$.
\end{itemize} 

\paragraph{SPiRL}
\begin{itemize}
    \item Discount $\gamma$: $0.97$. Tuned over $[0.90, 0.95, 0.97, 0.99]$.
    \item KL weight $\alpha$: 0.01. Tuned over $[0.005, 0.01, 0.05]$.
    \item Actor-network architecture: MLP [128, 256]
\end{itemize}

\paragraph{\alg}
      \begin{itemize}
        \item Discount $\gamma$: $0.97$. Tuned over $[0.90, 0.95, 0.97, 0.99]$. 
        \item Prior training: we use the parameters reported in Table \ref{table:hyperparams-prior}.
      \end{itemize}

\subsection{Hardware experiments}
\label{app:hardware_exp}
The experiments are carried out on a real desktop setting depicted in Figure \ref{fig:method} (upper right).
We use \textit{Boston Dynamics} Spot robot \cite{noauthor_spot_nodate} for all our experiments.
The high-level planner, trained in simulation, is used at inference time to predict the required sequence of motion primitives to achieve a desired goal. The motion primitives are executed using established motion planning methods \cite{zimmermann2022dca} to create optimal control trajectories. 
We then send desired commands to Spot using \textit{Boston Dynamics Spot SDK}, a Python API provided by Boston Dynamics, which we interface with using \textit{pybind11}.
We refer the reader to the Video material for a visualization of the real-word experiments.

\subsection{Optimality}
\label{app:optimality}
Let us consider an MDP $\mathcal{M}$, the reduced MDP $\mathcal{M'}$ as defined in Subsection \ref{ssec:learn_maskedQ1}, the selection operator  ${\alpha: \mathcal{S \to P(A)}}$ and a set of feasible policies under $\alpha$: $\Pi_\alpha=\{ \pi | \pi(s,g) \in \alpha(s) \forall s \in \mathcal{S}, g \in \mathcal{G} \}$.

\begin{theorem}
Given an MDP $\mathcal{M}$, the reduced MDP $\mathcal{M'}$ and a selection operator $\alpha$ such that the optimal policy for $\mathcal{M}$ belongs to $\Pi_\alpha$ (that is $\pi_{\mathcal{M}}^*(s, g) \in \Pi_\alpha$),
the optimal policy in $\mathcal{M'}$ is also optimal for $\mathcal{M}$, i.e. $\pi_{\mathcal{M'}}^*(s, g) = \pi_{\mathcal{M}}^*(s, g)$.
\end{theorem}

\begin{proof}
We can define a goal-conditioned value function $V^\pi(s,g)$, defined as the expected sum of discounted future rewards if the agent starts in $s$ and follows policy $\pi$ thereafter: ${V^\pi(s, g) = \mathbb{E}_{\mu^\pi}\left[\sum_{t-1}^\infty \gamma^{t-1} R(s, g) \right]}$ under the trajectory distribution ${\mu^\pi(\tau|g) = \rho_0(s_0) \prod_{t=0}^\infty P(s_{t+1}|s_t, a_t)}$ with $a_t = \pi(s_t, g) \; \forall t$.
By construction, the transition kernel $P$ and the reward function $R$ are shared across $\mathcal{M}$ and $\mathcal{M'}$, and a feasible policy $\pi$ always selects the same action in both MDPs, thus trajectory distributions and value functions are also identical for feasible policies. More formally, if $\pi \in \Pi_\alpha$, then $\mu^\pi_{\mathcal{M}}=\mu^\pi_{\mathcal{M'}}$ and $V^\pi_{\mathcal{M}}(s, g) = V^\pi_{\mathcal{M'}}(s, g)$.

It is then sufficient to note that 

\begin{equation}
\begin{aligned}
     \pi_{\mathcal{M'}}^*(s, g)&\stackrel{def}{=}\argmax_{\pi \in \Pi_\alpha} V^{\pi}_{\mathcal{M'}}(s, g)=\argmax_{\pi \in \Pi_\alpha} V^{\pi}_{\mathcal{M}}(s,g)=\\
     &=\argmax_{\pi \in \Pi} V^{\pi}_{\mathcal{M}}(s, g)\stackrel{def}{=}\pi_{\mathcal{M}}^*(s, g),
\end{aligned}
\end{equation}

where the second equality is due to the previous statement, and the third equality is granted by the assumption that the optimal policy for $\mathcal{M}$ is feasible in $\mathcal{M'}$ (i.e. $\pi_\mathcal{M}^* \in \Pi_\alpha$).
\end{proof}

Hence, our proposed algorithm, which allows us learning directly on the reduced MDP $\mathcal{M'}$, can under mild assumptions retrieve the optimal policy in the original $\mathcal{M}$.

We finally remark that model-free PAC-MDP algorithms have shown to produce upper bounds on sample complexity that are $\tilde O(N)$, where $N \leq |\mathcal{S}||\mathcal{A}|$ \cite{lattimore2012pac}, i.e.,  directly dependent on the number of state-action pairs. 
Hence, learning in the reduced MDP $\mathcal{M'}$ instead of $\mathcal{M}$ could lead to near-linear improvements in sample efficiency as the number of infeasible actions grows. We demonstrate it empirically in the subsection \textit{Sample Complexity} in \ref{sec:Experiments}.

\end{document}